\documentclass[accepted]{uai2026} 
                        
\usepackage[american]{babel}

\usepackage{natbib} 
\usepackage{mathtools} 
\usepackage{booktabs} 
\usepackage{tikz} 
\usepackage{comment}

\newcommand{\bm}[1]{\mathbf{#1}}
\newcommand{\w}[1]{\mathcal{#1}}
 
\newcommand{\bs}[1]{\boldsymbol{#1}}

\newcommand{\pr}[1]{\text{Pr}\left\{ {#1}\right\}}
\newcommand{\T}{\intercal}
\usepackage{bbm}
\usepackage{subfigure}
\usepackage{natbib}
\usepackage{thm-restate}
\usepackage{amsthm}

\newtheorem{lemma}{Lemma}
\newtheorem{proposition}{Proposition}
\usepackage{amsfonts}
\usepackage{amssymb}
\DeclareMathOperator*{\argmax}{argmax}

\title{Exploiting Concavity Information in Contextual Bandit Optimization}

\author[1]{\href{mailto:<kevin.li1324@gmail.com>?Subject=Your UAI 2026 paper}{Kevin Li}{}}
\author[1]{Eric Laber}
\affil[1]{%
    Department of Statistical Science.\\
    Duke University\\
    Durham, North Carolina, USA
}
  
\begin{document}
\maketitle

\begin{abstract}

The contextual bandit models sequential decision-making problems in which rewards depend on both the chosen action and observed context. In many domains such as medicine, business, and engineering, prior knowledge provides structural information about the reward function that can be exploited to improve optimization efficiency. This paper studies settings where, for each observable context, the conditional mean reward is known to be concave with respect to the action variable. To leverage this structure, we develop a contextual bandit algorithm that conditions a Bayesian Gaussian Process posterior on concavity constraints. We propose a novel reward model that combines a concavity-preserving regression spline basis with a constrained GP posterior, yielding a tractable shape-constrained estimator. Building on this model, we construct a UCB algorithm and establish new posterior concentration inequalities for the constrained posterior, leading to regret guarantees that are never worse than applying standard GP-UCB without concavity information and can be strictly tighter when concavity is informative. Experiments on benchmark problems and a Warfarin dosing test application demonstrate substantial reductions in cumulative regret relative to state-of-the-art baselines.
\end{abstract}

\section{INTRODUCTION}

The contextual bandit framework is used for sequential 
optimization in settings where the reward (utility) 
of any decision depends on auxiliary 
context variables 
\citep[]{chu2011contextual, bouneffouf2020survey,silva2022multi}. 
Any efficient sequential optimization algorithm makes use
of prior knowledge to impose structure on the considered class
of reward models. In the case of contextual bandits,
such structure often comes in the form of assumptions on
the conditional mean reward given context
and action.  
Popular assumptions include linearity \citep{chu2011contextual}, samples from 
a Gaussian Process (GP) \citep{srinivas2009gaussian, krause2011contextual}, 
smoothness of the function defined by bounds on RKHS norm 
\citep{valko2013finite, chowdhury2017kernelized}, or effective dimensionality 
as derived from the tangent kernel of a deep network 
\citep{zhou2020neural, zhang2021neural, kassraie2022neural}.

 Within medicine and business applications, the reward function is often known to be concave with respect to the action variable at each context value. In biology and pharmacology, this recurring pattern in the reward function is widely known as a U-shaped, bi-phasic dose-response curve \citep{calabrese2001u}. At low doses, an agent often provides beneficial function for a biological system. However, at a certain dosage level, benefits become saturated and increased dosage causes harmful effects. This phenomenon results in concave dose-response curves that have been studied extensively for decision making in cancer treatment \citep{portier1998u, reynolds2010potential}, environmental management/public health \citep{portier1998u, davis1990u}, psychological/behavioral intervention \citep{baldi2005inverted}, and dose optimization in personalized medicine \citep{chen2016personalized}. In business and economics, concave utility functions are foundational structural assumptions for decision making and modeling. A well-known example occurs in 
bidding and pricing systems in 
commerce, where 
increasing
price increases the per-unit profit but decreases overall sales volume
\citep{ meyer1968consistent, varian1982nonparametric, carroll1996concavity, boehm2022convex}.

This paper shows how incorporating this concavity information into the GP contextual bandit algorithm can improve 
bandit optimization performance. GP contextual bandits have a key advantage of being effective, yet simple to implement. Therefore they have been successfully applied in a wide variety of applications such as experimental design in biology \citep{durand2018contextual}, database routing optimization \citep{cereda2021cgptuner}, and industrial design \citep{zhou2025gaussian}.

Independent of bandit optimization, significant research has shown the advantages of incorporating convexity/concavity 
constraints in function estimation \citep{JMLR:v14:hannah13a, wang2016estimating, 
mazumder2019computational,aubin2020hard}. However, it is challenging to design bandit 
algorithms with provable sub-linear regret bounds from these bespoke models.

This paper introduces Concave Spline Gaussian Process Upper Confidence Bound (CSGP-UCB), a contextual bandit algorithm that leverages known concavity of the mean reward in the action (for each fixed context) to accelerate sequential optimization while retaining sub-linear cumulative regret guarantees. By pairing a concavity-preserving regression-spline GP with an exact translation of concavity into finite linear inequality constraints on spline coefficients, the approach induces a truncated (shape-conditioned) GP posterior that remains computationally and theoretically tractable. A central technical contribution is a new set of concentration inequalities for linear functionals of this truncated posterior, enabling UCB confidence bounds that are never wider than those from unconstrained GP-UCB and can be substantially tighter when concavity is informative. These results yield (to the author's knowledge) the first sub-linear regret guarantees for a GP contextual bandit that explicitly conditions on concavity information under Bayesian assumptions. Empirically, CSGP-UCB achieves markedly lower regret than a broad set of state-of-the-art baselines across benchmark functions, simulations, and a Warfarin dosing test application, highlighting the value of concavity constraints in the moderate-sample regimes that motivate many real applications.

\section{BACKGROUND}
\subsection{Concave Contextual Bandit}

We consider a sequential decision making problem consisting of $T$ rounds. 
Our goal is to maximize the expected cumulative reward. In each round, the 
decision maker 
receives a $d$-dimensional context $\bm{x}_t \in \w{X}$ and must choose 
$a_t \in \w{A} \subseteq [0, 1]$. After choosing the action $a_t$, 
the system produces a noisy reward
\begin{align}
\label{eqn:obs_model}
y_t = f(a_t, \bm{x}_t) + \epsilon_t,
\end{align}
where $\epsilon_t \sim_{\mathrm{iid}} N(0, \sigma^2)$. 
Thus, $f(a_t, \bm{x}_t)$ represents the mean reward given action $a_t$ and
context $\bm{x}_t$. We further assume that $f(a, \bm{x})$ is twice continuously differentiable at every $a \in \mathcal{A}$ for each $\bm{x}\in\mathcal{X}$. 
At each time $t$, the decision maker also receives 
concavity information such that 
for each $t'\le t$ the second partial with respect to $a$ satisfies 
$\frac{\partial^2 f(a, \bm{x}_{t'})}{\partial a^2} \leq 0$ for all $a \in [0,1]$, i.e., the map $a \mapsto f(a, \bm{x}_{t'})$ is concave with respect to $a$ 
at the current and all previously observed contexts.
This assumption is natural when the decision maker only wants to condition on concavity 
conditions for contexts $\bm{x}$ that are feasible for the system and the feasible region is 
not known {\em a priori}. 

The action selected 
at time $t$ can be evaluated by the regret of not 
choosing the optimal action given the context at time 
$t$: $r_t \triangleq 
\sup_{a \in \w{A}} f(a, \bm{x}_t) - f(a_t, \bm{x}_t)$. After the $T$ rounds, 
the cumulative regret is $R_T \triangleq 
\sum_{t=1}^T r_t$. Our goal is to construct an algorithm that effectively 
makes use of concavity information to improve performance in small samples 
while also 
ensuring sub-linear regret as $T$ grows large. 

\subsection{Gaussian Processes}

A Gaussian Process $h(\bm{x})$ is a stochastic process for which its evaluation on any finite 
subset of inputs follows a Multivariate Normal distribution \citep{Rasmussen2006Gaussian}. A 
Gaussian process is fully defined by its mean function $\mu(\bm{x})$ and covariance function 
$k(\bm{x}, \bm{x}')$. 
A convenient property of Gaussian Processes is that upon observing some 
realizations of a Gaussian Process 
$\bm{h} = \left\lbrace 
h(\bm{x}_1) \dots h(\bm{x}_n)\right\rbrace^\T$, the posterior 
distribution of unobserved samples $\bm{h}^{*} = 
\left\lbrace h(\bm{x}_1^{*}), \dots, h(\bm{x}_m^{*})\right\rbrace^{\T}$ 
follows a 
Multivariate Normal distribution with mean $\bs{\mu}^{*}$ and
covariance $\bm{V}^{*}$ defined by  
\begin{align}
\label{eqn:gp_post}
\bs{\mu}^{*} = \bs{\mu}_{h^{*}} + \bm{K}_{m, n} \bm{K}_{n, n}^{-1} ( \bm{h} - \bs{\mu}_h), \\
\bm{V}^{*} = \bm{K}_{m, m} - \bm{K}_{m, n} \bm{K}_{n, n}^{-1} \bm{K}_{n, m} \notag,
\end{align}
where: 
$\bm{K}_{n, n}$ is the covariance matrix of the observed realizations $\bm{h}$ such 
that $(i, j)$th entry is $k(\bm{x}_i, \bm{x}_j)$;  
$\bm{K}_{m, m}$ is the covariance matrix of the unobserved realizations; 
$\bm{K}_{m, n}$ is the cross covariance matrix such that the $(i, j)$th entry is $k(\bm{x}_i^{*}, \bm{x}_j)$; and $\bs{\mu}_{h}, \bs{\mu}_{h^{*}}$ are the mean function evaluated on the observed and unobserved samples respectively.

\subsection{Concave Spline Regression}
\label{sec:csplines}
For each context, the mean reward is a concave function of the action. Thus,
it is useful to have a mechanism for modeling a one-dimensional 
function, say $g(a)$, that is expressive but also easily constrained 
to be concave. In this section, we discuss modeling such a function
using regression splines; we omit dependence on the context and show 
how to make the overall mean function smooth across contexts in
Section \ref{subsec:reward_model}. We choose regression splines due to their representative power. Given enough knots, our chosen constrained spline model can uniformly approximate continuous concave functions defined on a closed interval \citep{de1978practical, meyer2008inference}.

For regression splines of polynomial order $k$, with $l$ knots we can derive the $l+k$ 
piece-wise polynomial M-Spline basis functions \citep{ramsay1988monotone} 
$M_j(a) \ , j = 1, \dots, l +k$ such that $M_j(a) \geq 0$. 
We model $g''(a)$ using these M-Spline basis functions
\begin{align*}
g''(a) = \sum_{j=1}^{l+k} M_j(a) \beta_j.
\end{align*}
A crucial property of the M-Spline basis is that for $k \leq 4$, $g''(a) \leq 0$ if and only if $\beta_j \leq 0$ for $j = 1, \dots, l + k $ \citep{ramsay1988monotone}.
Integrating each M-Spline basis function $M_j(a)$ twice yields the C-Spline Basis functions introduced by \citet{meyer2008inference}:
$C_j(a) = \int_{0}^{a} \int_{0}^u M_j(s) ds du$.
We consider a model for $g(a)$ 
that uses these C-Spline basis functions 
along with the constant and identity functions:
\begin{align}
\label{eqn:splinemodel}
g(a) = \sum_{j=1}^{l+k} C_j(a) \beta_j + \beta_{l+k+1}a + \beta_{l + k + 2}.
\end{align}

From the properties of $g''(a)$, it follows that $g(a)$ is concave if and only 
if $\beta_j \leq 0$ for $j = 1, \dots, l + k$. 
Thus the resulting constrained spline function can represent all concave piecewise polynomials of order $k+2$.

\section{Related Work}

Our work is most closely related to the literature of Gaussian Process Bandits which were originally introduced by \citet{srinivas2009gaussian} and extended to the contextual case in \citet{krause2011contextual}. In a similar 
vein, \citet{valko2013finite} developed a Kernelized UCB algorithm where they bound 
regret using the effective dimension of the kernel. \citet{chowdhury2017kernelized} derived improved bounds for context-free GP/Kernelized UCB bandits and the first frequentist regret 
bounds for GP Thompson sampling.

Previous GP-bandit literature has incorporated complex structural assumptions into the prior of the reward function. The strategy has been to directly condition on the event that a sample from a standard GP satisfies the assumption. \cite{kandasamy2019multi} develops an algorithm conditioned on GP-bandits with a multi-fidelity structure. \cite{scarlett2018tight} conditions on constraints to derive Bayesian optimization cumulative regret bounds. The methods presented in these papers do not translate well to our goal. In their methodology, conditioning on the structural constraints inflates the regret bounds relative to the unconditioned case. In addition, it is not clear how to practically condition a standard GP covariance function on infinite dimensional concavity constraints. Both outcomes are unsatisfactory for our goal of using concavity information to achieve more efficient optimization.

 Independently of the bandit literature, the problem of shape-constrained function estimation in the form of monotone/convex GP regression has been extensively studied. One strategy has been to approximate the assumption by only conditioning on the constraints within a discrete grid of the input space \citep{da2012gaussian, wang2016estimating}. These methods quickly become impractical, as conditioning on higher derivatives is notoriously numerically unstable \citep{he2018instability} and the grid size must grow exponentially with dimension. \cite{maatouk2017gaussian} propose a more similar approach to ours. They tile the covariate space with a tensor grid basis of local basis functions with independent Gaussian priors on the coefficients. Unfortunately, their tensor grid grows exponentially with context dimension and does not scale to contexts with $d >2$. Our paper extends their approach for the contextual bandit setting by allowing the coefficients for each basis function to vary smoothly with the context variables and showing that this leads to sub-linear cumulative regret in the bandit optimization context.

\section{CONCAVE SPLINE GP CONTEXTUAL BANDIT}
\subsection{Reward Model}
\label{subsec:reward_model}

Conditional on each context, we assume that the expected 
reward function takes the form of a regression spline as described in equation 
\ref{eqn:splinemodel} with the generalization that the spline regression coefficients vary as smooth functions of $\bm{x}$. 
We define the CSGP model for the reward as:
\begin{align}
\label{eqn:function_model}
f(a, \bm{x}) &= \sum_{j=1}^{J-2} C_j(a) \bs{\beta}_j(\bm{x}) + \bs{\beta}_{J-1}(\bm{x}) a + \bs{\beta}_{J}(\bm{x})\\ 
& \notag := \bs{\phi}(a)^T \bs{\beta}(\bm{x}), \notag
\end{align}
where each $C_j(a)$ is the C-Spline basis function described in Section \ref{sec:csplines}, $\bs{\beta}(\bm{x}) := 
\left\lbrace 
\bs{\beta}_1(\bm{x}), \dots, \bs{\beta}_J(\bm{x})\right\rbrace^T$ 
and $\bs{\phi}(a) := \left\lbrace 
C_1(a), C_2(a), \dots ,a, 1\right\rbrace^T$. 
We place Gaussian Processes priors on $\bs{\beta}_j(\bm{x})$ with mean function $\bs{\mu}_j$ and covariance function $k_{j}(\bm{x}, \bm{x}')$. These priors induce a Gaussian Process on $f(a, \bm{x})$.

\subsection{Incorporating Concavity Information}
To incorporate concavity information and complete the CSGP formulation, we assume that at each time step 
$t$ we observe the context $\bm{x}_t$ and the information 
that $f(a, \bm{x}_t)$ is concave in $a$ at all $\bm{x}_1 \dots \bm{x}_t$. This form of concavity constraint is natural when the decision maker only wants to condition on concavity 
conditions for contexts $\bm{x}$ that are feasible for the system and the feasible region is 
not known exactly {\em a priori}. Under our choice of spline basis (see Section \ref{sec:csplines}) this 
information is exactly equivalent to conditioning on 
$\bs{\beta}_j(\bm{x}_{t'}) \leq 0$ for $j = 1 \dots J -2, t' \leq t$. 
 An important property of 
our reward and observation models is that after conditioning on such 
information, the posterior for the finite realizations of $\bs{\beta}
(\bm{x})$ is a truncated Gaussian. 
Let $\mathbb{R}_e = \mathbb{R}\cup \left\lbrace \infty \right\rbrace$
denote the (positively) extended reals. Given 
$\bs{\nu}\in\mathbb{R}_e^p$, $\bs{\omega}\in\mathbb{R}^p$, and
$\bs{\Omega}\in\mathbb{R}^{p\times p}$, we let 
 $N_{\bs{\nu}}(\bs{\omega}, \bs{\Omega})$ denote a multivariate
 normal distribution with mean $\bs{\omega}$ and covariance
 $\bs{\Omega}$ that has been truncated above by $\bs{\nu}$. 
The following result is proved in Appendix \ref{proof:beta_posterior}. 

\begin{restatable}{lemma}{truncposterior}\label{lemma:trunc_lemma}
Suppose that we observe $\bm{y}_{t-1} = \left(y_1, \ldots, y_{t-1}\right)^T$. 
Define $\bs{\beta}_{t} := 
\left\lbrace 
\bs{\beta}(\bm{x}_1), \dots, \bs{\beta}(\bm{x}_{t}) 
\right\rbrace^T$ 
to be the concatenated vector of coefficients from every context up to time $t$. Let $\w{C}_{t'}$ denote the event that 
$\bs{\beta}_j(\bm{x}_{t'}) \leq 0$, for 
$j = 1, \ldots, J-2$. 
Then the time $t$ posterior $p\left(\bs{\beta}_t| \bm{y}_{t-1}, \cap_{t' =1}^{t} \w{C}_{t'} \right)$ is the upper truncated Multivariate Normal distribution $N_{\bs{\nu}}(\bs{\mu}_t, \bs{\Sigma}_t)$, 
where $(\bs{\mu_t}, \bs{\Sigma}_t)$ are the posterior quantities obtained from the standard GP time $t$ posterior without truncation, and $\bs{\nu}_j = 0, j \leq J-2$ and $\bs{\nu}_j = \infty$ otherwise. 
\end{restatable}

Proving regret bounds under our 
model requires quantifying the posterior 
uncertainty surrounding the 
expected reward. The mean reward function can be expressed as $f(a, \bm{x}_t) = \bm{c}_t(a)^T 
\bs{\beta}_t$, where 
 $\bm{c}_t(a) \in \mathbb{R}^{Jt}$ is
 defined so that the last $J$ elements of 
$\bm{c}_t(a)$ equal to $\bs{\phi}(a)$ while 
its remaining entries are zero. Therefore, our time $t$ CSGP posterior for $f(a, \bm{x}_t)$ is a Linear Combination of a Truncated Multivariate Normal (LCTMVN). We refer to this random variable as LCTMVN throughout the paper.

Without conditioning on concavity information the posterior distribution of $f(a, \bm{x}_t)$ is Gaussian with variance:
\begin{align}
\label{eqn:vanilla_var}
&\sigma^2_{\bm{c}_t}(a, \bm{x}_t) = \bm{c}_t(a)^T \bs{\Sigma}_t \bm{c}_t(a) = \\ 
&k_f\left\lbrace (a, \bm{x}_t), (a, \bm{x}_t) \right\rbrace - 
\bm{k}_{t, t-1} (\bm{K}_{t-1, t-1} + \sigma^2 \bm{I})^{-1} \bm{k}_{t, t-1}, \notag
\end{align}
where $\bm{k}_{t, t-1} \in \mathbb{R}^{1 \times t-1}$ and the $i$th element is 
$k_f\left\lbrace (a, \bm{x}_t), (a_{i}, \bm{x}_i) \right\rbrace$ 
and $\bm{K}_{t-1, t-1}$ is the prior covariance matrix of the $\{f(a_{t'}, \bm{x}_{t'})\}_{t' = 1}^{t}$.

The challenge for regret analysis is that the concavity conditioned LCTMVN posterior has no canonical form or theoretically amenable quantile function. To derive regret guarantees, we will instead rely on novel concentration results. 

\subsection{LCTMVN Concentration Properties}
\label{sec:concentration}

Effective UCB bandit algorithms require confidence bounds that are narrow as possible while maintaining sufficiently high coverage probability of the true reward function. In this section, we derive concentration inequalities for concavity-conditioned LCTMVN variables that produce a composite confidence bound with the same coverage guarantees as the standard un-truncated Gaussian case, but that is never wider—and is often strictly narrower in practice. To our knowledge, these concentration results are new and may be of independent interest for the Bayesian constrained regression community. Proofs can be found in Appendix \ref{sup:conc_proofs}.

We first prove coverage guarantees for confidence bounds centered around the LCTMVN posterior mean. The degree of concentration depends on  Bregman divergences that are sensitive to the geometry of the truncation.

\begin{restatable}{proposition}{meanConcBregman}
\label{prop:mean_concentration}
Assume the concavity conditioned posterior of our coefficient vector $\bs{\beta}$ at time $t$ is $\bs{\beta}_t \sim N_{\bs{\nu}}(\bs{\mu}_t, \Sigma_t)$. Let $\bs{\mu}_t^{*} = \mathbb{E}(\bs{\beta}_t |\bm{y}_{t-1}, \cap_{t' =1}^{t} \w{C}_{t'})$. Then given $b \geq 0$:
\begin{align*}
\pr{| \bm{c}_t(a)^T\bs{\beta}_t - \bm{c}_t(a)^T \bs{\mu}_t^{*}| \geq b } 
\leq &\exp\left(-\frac{b^2}{2\sigma_{\bm{c}_t}^2(a, \bm{x}_t)}\right) \\
 &\times P_{t}(a, b)
\end{align*}
Where 
\begin{align*}
P_{t}(a, b) &= \Biggl\{ \exp\left[ -D_{G, \bm{c}_t(a)}^{\bs{\mu}_t, \bs{\Sigma}_t}\left(\bs{\mu}_t + \frac{b}{\sigma_{\bm{c}_t}^2(a, \bm{x}_t)} \Sigma_t \bm{c}_t(a), \bs{\mu}_t\right) \right] \\
&  + \exp\left[ -D_{G, \bm{c}_t}^{\bs{\mu}_t, \bs{\Sigma}_t}\left(\bs{\mu}_t - \frac{b}{\sigma_{\bm{c}_t}^2(a, \bm{x}_t)} \Sigma_t \bm{c}_t(a), \bs{\mu}_t\right) \right] \Biggr\} 
\end{align*}

$D_{G, \bm{c}}^{\bs{\mu}_t, \bs{\Sigma}_t}(\cdot, \cdot)$ represents the Bregman divergence with respect to the convex function $G(\bs{\mu}) = -\log\left\{ \int_{\bs{\beta} \preccurlyeq \nu} p_{\bs{\mu}, \Sigma_t}(\bs{\beta}) d\bs{\beta} \right\}$. Here $p_{\bs{\mu}, \Sigma_t}(\bs{\mu})$ is the standard multivariate normal density with mean $\bs{\mu}$ and covariance $\Sigma_t$.
\end{restatable}

The resulting inequality differs from the standard Gaussian concentration inequality (without concavity information/truncation) in the GP bandit literature by the factor $P_{\bm{c}_t}^{t}(a, b)$ \citep{srinivas2009gaussian}. Interpreting the Bregman divergences, we can see that $P_{\bm{c}_t}^{t}(a, b) < 1$ when the truncation causes curvature for $G(\bs{\mu})$ in the direction $\pm \Sigma_t \bm{c}_t(a)$. As $G(\bs{\mu})$ represents the probability of truncation region as a function of the mean parameters,  $P_{\bm{c}_t}^{t}(a, b)$ represents the severity of the truncation effect on the LCTMVN variable. When the data $\bm{y}_t$ is not sufficient to concentrate the posterior in the truncation region, the LCTMVN bounds will have higher coverage probability than the standard Gaussian bounds for the same width. In Lemma \ref{lemma:myucb}, we show that $P_{\bm{c}_t}^t(a, b) < 1$ implies the existence of a mean-centered confidence bound with width parameter $b' <b$ that achieves the same coverage probability as the standard Gaussian confidence bound with width parameter $b$. 

\begin{figure*}
    \centering
    \includegraphics[width = .7\linewidth]{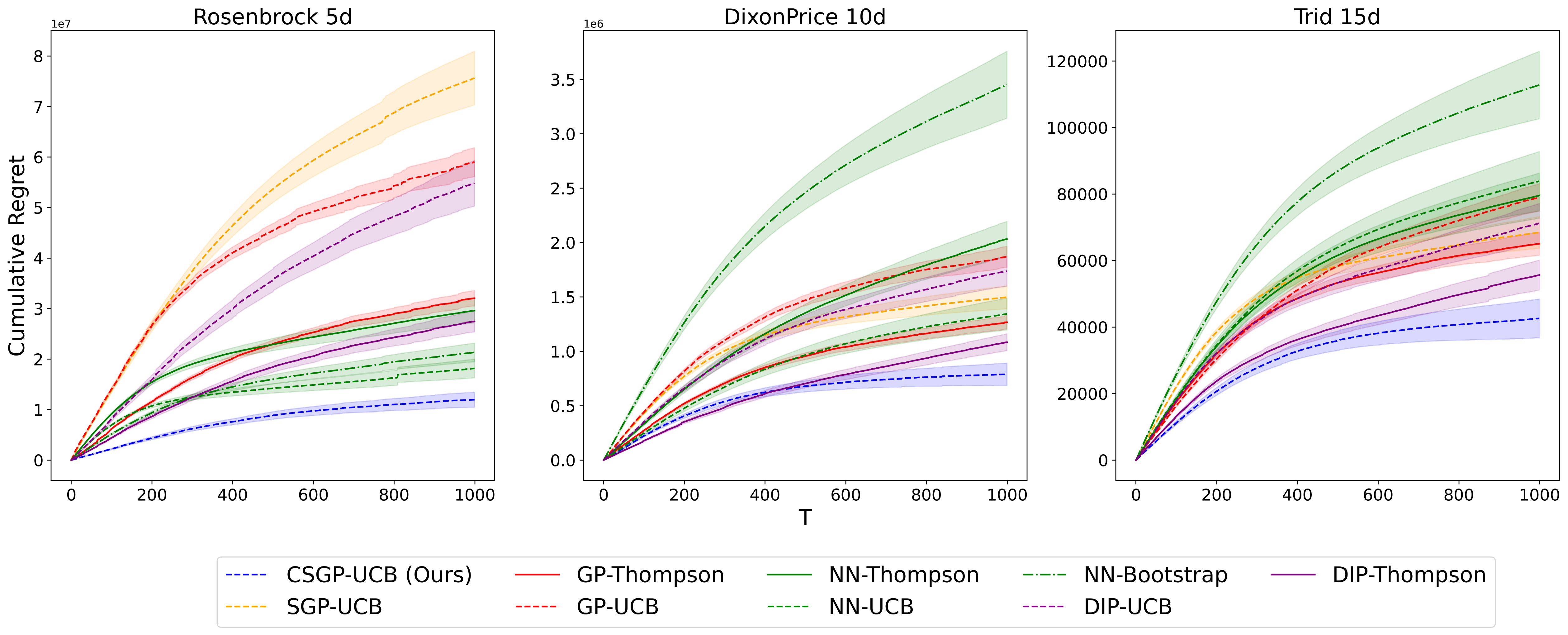}
    \caption{Mean cumulative regret for contextual bandit methods on benchmark functions across $T$, with $\pm$ 1 standard error bars. }
    \label{fig:synthetic_results}
\end{figure*}

It is possible that $P_{\bm{c}_t}^{t}(a, b) > 1$, in which case the inequality becomes looser than the standard (unconstrained) Gaussian inequality. As shown in Appendix~\ref{sec:conc_counter}, this is unavoidable: LCTMVN variables can exhibit weaker mean-centered concentration than a Gaussian. To overcome this issue, our algorithm will also use confidence bounds centered around the posterior median that match the coverage guarantees achieved by the standard Gaussian case. We use tools  from optimal transport \citep{caffarelli2000monotonicity, kim2012generalization} and Gaussian isoperimetric  inequalities \citep{ledoux2006isoperimetry} to derive the following result.

\begin{restatable}{proposition}{medCoProp}\label{prop:median_concentration}

    Suppose $\bs{\beta}_t \sim N_{\bs{\nu}}(\bs{\mu}_t, \Sigma_t)$. Let $m_t(a, \bm{x}_t)$ denote the median of $\bm{c}^T(a)\bs{\beta}_t$. Then
\[
    \pr{  | \bm{c}(a)^T\bs{\beta}_t - m_t(a, \bm{x}_t)| \geq b } \leq \exp\left(\frac{-b^2}{2\sigma^2_{\bm{c}_t}(a, \bm{x}_t)} \right)
\]
\end{restatable}

Combining the two concentration inequalities allows us to create a LCTMVN confidence bound that achieves the same concentration as the Gaussian case is at least as narrow. The following lemma designs this confidence bound with corresponding coverage guarantee:

\begin{restatable}{lemma}{combinedbound} \label{lemma:myucb}
Suppose $\bs{\beta}_t \sim N_{\bs{\nu}}(\bs{\mu}_t, \Sigma_t)$ and Let $\chi_t$ be a predefined constant. Define 
    \[
\kappa_t(a, \bm{x}_t) = \begin{cases}
                \bm{c}_t(a)^T\bs{\mu}_t^{*} , & \text{if $P_t\left[a, \chi_t\sigma_{\bm{c}_t}(a, \bm{x}_t)\right] < 1$},\\
                 m_t(a, \bm{x}_t), & \text{else}.
                \end{cases}.
\]
\[
\alpha_t(a, \bm{x}_t) = \begin{cases}
                \chi_t^{*}(a, \bm{x}_t) , & \text{if $P_t\left[a, \chi_t\sigma_{\bm{c}_t}(a, \bm{x}_t)\right] < 1$},\\
                \chi_t, & \text{else}.
                \end{cases}.
\]
\begin{align*}
&\chi_t^{*}(a, \bm{x}_t)  
\\ &:= \inf_{s} 
\left\{s: \exp\left(-\frac{s}{2}\right) P_t\left[a, s \sigma_{\bm{c}_t}(a, \bm{x}_t)\right] \leq \exp\left(-\frac{\chi_t}{2}\right) \right\}.
\end{align*}
$\chi_t^{*}(a, \bm{x}_t)$ exists and we have the following concentration bound:
\begin{align*}
    &\pr{| f(a, \bm{x}_t)  - \kappa(a, \bm{x}_t)| \geq \alpha_t(a, \bm{x}_t) \sigma_{\bm{c}_t}(a, \bm{x}_t)} 
    \\ &\leq \exp\left(-\frac{\chi_t}{2}\right).
\end{align*}
\end{restatable}

In practice, we numerically compute $\chi_t^{*}$ which we discuss in Appendix \ref{sup:computation}. We investigate the empirical behavior of $\alpha_t(a, x_t)$ in section \ref{sec:confidence}. 

\subsection{UCB Algorithm}
We now define our CSGP-UCB algorithm. Recall that $\sigma^2_{\bm{c}_t}(a, \bm{x}_t)$ is the variance of the unconstrained CSGP posterior defined in equation \ref{eqn:vanilla_var}. Let $\kappa(a, \bm{x}_t), \alpha(a, \bm{x}_t)$ be defined as in Lemma \ref{lemma:myucb} using $\bs{\mu}_t, \Sigma_t$ from the time $t$ concavity conditioned CSGP posterior. Given a predefined sequence $\chi_t$, our CSGP-UCB algorithm selects action $a_t$ as
\[
a_t = \argmax_{a \in \mathcal{A}} \kappa_t(a, \bm{x}_t) + \alpha_t(a, \bm{x}_t)^{\frac{1}{2}} \sigma_{\bm{c}_t}(a, \bm{x}_t)
\]
Our UCB algorithm chooses the narrower of the mean and median centered confidence bounds for each  $a \in \mathcal{A}$. If $P_t(a, \chi_t) <1$, then we will use the mean bound as it has a smaller width. Otherwise the median centered bound will have smaller width. The sequence $\chi_t$ is chosen as a gradually increasing sequence to ensure that the bound covers the true value of $f$ with high probability. We then select the action $a$ that has the highest bound.

Computing the bound requires computing the LCTMVN variable means ($\bs{\mu}^{*}$), medians ($m(a, \bm{x}_t)$), log orthant probabilities ($G(\bs{\mu})$), and optimal confidence parameter $\chi_t^{*}$. We discuss computation of these quantities in the Appendix \ref{sup:computation}.

\subsection{Regret Analysis}
We present the regret analysis for our UCB algorithm under the Bayesian assumption that the 
expected reward function $f(a, \bm{x})$ is a sample from CSGP prior described in Section \ref{subsec:reward_model}. This assumption is commonly made and analyzed in the bandit literature \citep{srinivas2009gaussian, krause2011contextual, chen2012joint, wang2015optimization, bogunovic2016time, kathuria2016batched, kassraie2022neural}. Proofs are contained in the Appendix.

To our knowledge, we provide the first sub-linear regret guarantees under the Bayesian assumption and the scenario that the bandit-algorithm receives concavity information. Any GP-posterior
that is not conditioned on this concavity information will be
mis-specified. The authors are not aware of algorithms that achieve sub-linear regret under the Bayesian assumption and mis-specified posteriors. 

Our regret bound uses the maximum information gain concept first exploited for GP Bandit problems by \cite{srinivas2009gaussian}. We represent the maximal information gain by $\gamma_T$ and prove in Appendix \ref{sec:mig} that for our CSGP kernel, $\gamma_T = \w{O}\left\lbrace J\log(T)^{d+1}\right\rbrace$.\footnote{See Appendix \ref{sec:mig} for more discussion}

To highlight the source of concavity  advantage, we note how the confidence parameter $\alpha_t$ relates to the time-$t$ instantaneous regret $r_t$. The following lemma, which follows directly from Lemma 4.1 of \cite{ krause2011contextual}, provides the following high probability bound on $r_t$ for GP-UCB algorithms in discrete action spaces.

\begin{restatable}{lemma}{regretlemma}
    
For $t\ge 1$, with probability $\geq 1-\delta$, $r_t \leq 2 \alpha_t(a_t, \bm{x}_t)^{\frac{1}{2}} \sqrt{\gamma_T}$
\end{restatable}

Thus, lowering the confidence parameter directly lowers the bound on $r_t$.   Under standard GP-UCB, the confidence parameter does not depend on posterior quantities or the action/context. Thus $\alpha_t(a, \bm{x}_t) = \chi_t$.  Concavity information allows $\alpha_t(a, \bm{x}_t) \ll \chi_t$ when the truncation is severe. Intuitively, having an tighter confidence parameter prevents over-exploration of uncertain regions. This translates to the following cumulative regret bound. 

\begin{restatable}{proposition}{regret}
\label{prop:discrete_ucb}
Let the expected reward function $f$ be a sample from a known CSGP prior with noise model $N(0, \sigma^2)$. Also suppose that at each time step $t$ the bandit algorithm receives noisy reward $y_t$ as well as concavity information that $\bs{\beta}_j(\bm{x}_{t'}) \leq 0$ for $j = 1 \dots J- 2, t' \leq t$. Pick $\delta \in (0, 1)$ and suppose one of the following assumptions holds:
\begin{itemize}
  \item $\w{A}$ is finite with $\chi_t = 2\log( |\w{A}|t^2 \pi^2/6 \delta)$
  \item $\w{A} \subset [0, 1]$ is compact and convex and the CSGP kernel $k_f(\cdot, \cdot)$ satisfies the following bound on the GP sample paths.\footnote{
The condition on 
$k_f(\cdot, \cdot)$ is satisfied with a spline order 6 and kernels such as the common Gaussian kernel and the 
Matern kernel with $\nu > 2$.
} For some constants $\eta, \zeta$
  \begin{align*}
  \Pr \{ \sup_{\bm{v} \in \w{A} \times \w{X}} |\partial f/ \partial \bm{v}_j| > L \} \leq \eta \exp\left\lbrace 
  -(L/\zeta)^2\right\rbrace,
  \end{align*}
  for $\forall j \in \{ 1, \dots, d+1 \}$.
  And we choose 
  \begin{align*}
  \alpha_t = 2\log(2 \pi^2 t^2/3\delta) + 2\log\left\lbrace 
  t^2 \zeta \sqrt{\log\left(4 \eta/\delta \right)} \right\rbrace.
  \end{align*}
\end{itemize}
Define $\alpha_T^{*} = \max \{\alpha_t(a_t, \bm{x}_t)\}_{t=1}^T$ so that $\alpha_T^{*} \leq \chi_T$.
Running CSGP-UCB obtains a regret bound $\w{O}^{*}(\sqrt{T \gamma_{T} \alpha_T})$ with high probability. Precisely let $C_1 = \frac{8}{\log(1 + \sigma^{-2})}$ then:
\begin{align*}
\Pr\left\{ R_t \leq \sqrt{C_1 T \alpha_T^{*} \gamma_T} + 2 , \forall T \geq 1 \right\} \geq 1 - \delta,
\end{align*}
\end{restatable}

Our choices for $\chi_t$ are the default confidence parameters used by GP-UCB bandits \citep{srinivas2009gaussian, krause2011contextual}. As $\alpha_T^{*} \leq \chi_T$, our cumulative regret bounds under concavity conditions are never worse than the bound obtained for the case without concavity information. For moderate time horizons $T$, concavity is often highly informative and induces substantial posterior concentration, so $\alpha_t < \chi_t$ and $\alpha_T^{*} < \chi_T$. This directly shrinks the confidence bounds and translates into substantially lower cumulative regret. We explore the behavior of $\alpha_T^{*}$ in Figure~\ref{fig:alpha_t} and Section~\ref{sec:confidence}. As $T \to \infty$, the data $\bm{y}_t$ increasingly identifies the function's concavity and the posterior mass concentrates within the feasible (truncated) region. The benefits of concavity information fades, and the regret rate is expected to recover the standard Gaussian-case asymptotic regret. This behavior is consistent with the literature on convex/constrained regression which has found that the asymptotic learning rates of constrained methods match the asymptotic rates of their unconstrained counterparts \citep{geyer1994asymptotics, tsybakov2008nonparametric, kur2020convex}. We also note that an important line of future work is to determine whether a sequence of $\alpha_t$ can be developed such that $\alpha_T^{*} < \chi_T$.

\begin{figure*}
  \centering
  \includegraphics[width=.75\linewidth]{new_figs/Simulations_DIP.png}
  \caption{Cumulative regret for numerical simulation study. The
  proposed methods, CSGP-Thompson and CSGP-UCB, obtain 
  markedly lower
  regret than competing methods across a range of problem dimensions 
  and scalings.}
  \label{fig:Simulation_Study}
\end{figure*}

\subsection{Robustness to mild violations of concavity.}
The exact CSGP posterior conditions on the curvature coefficients satisfying
$\bs{\beta}_j(\bm{x}_{t'}) \leq 0$ for $j=1,\ldots,J-2$, which enforces action-wise concavity. When domain knowledge suggests that concavity may hold only approximately, this hard constraint can be relaxed to
\[
    \bs{\beta}_j(\bm{x}_{t'}) \leq c,
    \qquad j=1,\ldots,J-2,
\]
for a small scalar $c>0$. The posterior remains an upper-truncated Gaussian, so the posterior computation and confidence-bound construction are unchanged after replacing the truncation threshold $0$ by $c$. The choice $c=0$ recovers the exact concavity model, while $c>0$ permits controlled violations of concavity in the C-spline coefficient scale. Appendix~\ref{sup:robust_concavity} provides additional implementation details and empirical results under non-concave perturbations, showing that a modest positive value of $c$ can reduce performance degradation when the reward is only mildly misspecified.

\section{NUMERICAL EXPERIMENTS}

We provide an empirical validation of CSGP-UCB and examine the evolution of the confidence parameter $\alpha_T^{*}$. In section \ref{sec:benchmarks} we evaluate our method on widely used optimization benchmarks provided by \cite{simulationlib, balandat2020botorch}. In section \ref{sec:simu} we study the behavior of CSGP-UCB as we vary the dimensions/smoothness of the underlying reward function. Section \ref{sec:confidence} examines the evolution of CSGP confidence parameter $\alpha_T^{*}$. Finally in section \ref{sec:warfarin}, we test our method on a Warfarin dosing function application provided by \cite{chen2016personalized}.

We benchmark our algorithms against state of the art non-parametric bandit baselines based on GPs and Neural Networks. We compare to GP-UCB/TS \citep{krause2011contextual, russo2014learning}, neural contextual bandits (NN-UCB/NN-TS/NN-Bootstrap) \citep{kassraie2022neural, zhang2021neural, riquelme2018deep}, our CSGP model without conditioning on concavity (SGP-UCB), and GP-UCB/TS where the second derivative of the GP posterior is conditioned to be negative on a grid as proposed by \cite{da2012gaussian, wang2016estimating} (DIP-UCB/TS). Details for the hyper-parameters and implementations of the methods can be found in Appendix \ref{sup:exp_details}. As \citet{zhang2021neural, kassraie2022neural} provide 
algorithms only for a finite action space,
we test the bandits with $\w{A}$ 
defined as a uniform grid of size 100 on $[0, 1]$. 

\subsection{Benchmark Functions}
\label{sec:benchmarks}
We test CSGP-UCB on popular benchmark functions widely used in the optimization literature that also satisfy our concavity assumptions \citep{simulationlib, balandat2020botorch}. We consider the Rosenbrock ($d =5$), Dixon-Price ($d = 10$), and Trid ($d = 15$). To adapt these benchmarks to the contextual bandit application, we sample context vectors uniformly from the domains of $d-1$ variables and select actions on the held out dimension. For more details of implementation please see the Appendix \ref{sup:synthetic}. 

Figure \ref{fig:synthetic_results} shows that CSGP-UCB significantly outperforms other methods. SGP-UCB performs similarly to its GP-UCB counterpart which indicates that the performance of CSGP-UCB can be attributed to conditioning on the concavity information and not the novel spline model. We also see that the approximate shape constraints of DIP-UCB/Thompson outperforms the standard GP methods for smaller $T$. However, we also see that the DIP bandit converges to optimal actions more slowly than the standard GP methods. This behavior can be attributed to modeling inaccuracy caused by the known numerical instability caused by conditioning on higher derivative information \citep{he2018instability}. We suspect that this is why \cite{da2012gaussian, wang2016estimating} do not implement/demonstrate second order constraints in their works.  

\subsection{Numerical Simulations}
\label{sec:simu}

This section examines the performance of CSGP-UCB across problem difficulties by varying the context dimension and underlying smoothness of the reward function. We simulate functions of the form:
\begin{align*}
f(a, \bm{x}) = \sum_{s = 1}^{10} h_s(a) \eta_s(\bm{x})^2 + \gamma(x),
\end{align*}
where $h_s(a)$ is sampled from a Gaussian Process prior conditioned on 
$h_s''(a) \leq 0$ on a very fine grid in $[0, 1]$. This conditioning 
produces functions $h_s(a)$ that are effectively concave in the interval 
\citep{wang2016estimating}. We sample $\eta_s(\bm{x})$ and 
$\gamma(\bm{x})$ from independent Gaussian Processes equipped with 
isotropic Gaussian kernels. The concavity of $f(a, \bm{x})$ with respect 
to $a$ follows as it is a sum of functions concave in $a$. We sample the 
contexts uniformly such that $\bm{x} \in [0, 1]^d$. A crucial property 
of our simulation set up is that the CSGP model is mis-specified for 
this simulation; i.e., $f(a, \bm{x})$ does not take the form of piecewise polynomials for any $\bm{x}_t$. We sample the 
noisy rewards from $y = f(a, \bm{x}) + \epsilon$ where $\epsilon \sim 
N(0, 0.1)$. We initialize each bandit algorithm with $n = 25$ 
observations and run each algorithm for $T = 500$ iterations. 

We vary simulation scenarios to control the difficulty of 
function estimation and optimization. We vary context dimension and the 
length-scales of the kernels governing the GPs for 
$\eta_s(\bm{x}), \gamma(\bm{x})$. Smaller length-scales for GPs lead to 
reward function samples that vary more quickly and are thus more 
difficult to learn. We examine the performance of the various algorithms 
for $d = 5, 25, 50$ with length-scales $\theta = d/5, d/10, d/25, d/50$. 
Scaling the length-scale by dimension ensures the problem does not 
become intractable without massive sample sizes. 

Figure \ref{fig:Simulation_Study} shows the mean cumulative regret of 
 the various bandit algorithms across $n = 25$ replications for each 
 experimental setting along with 1 SE bars. CSGP-UCB increasingly outperforms competing benchmarks as the difficulty of the optimization task grows with respect to both length-scale and dimension. SGP-UCB frequently under-performs GP-UCB which indicates that the superior performance of CSGP-UCB is due to the concavity information and not due to the structure of the model/simulation. CSGP-UCB outperforms the approximately concavity conditioned DIP-UCB/Thompson, especially as length-scales become small. We suspect this is due to the fact that DIP-UCB/Thompson requires an increasingly fine mesh of concavity conditioned inducing points to be effective depending on the difficulty of the problem. However, the optimal grid is unknown before running the algorithm and may be unfeasible due to numerical instability issues \citep{he2018instability}. 

\subsection{Behavior of $\alpha_T^{*}$}
\label{sec:confidence}

\begin{figure}
    \centering
    \includegraphics[width=\columnwidth]{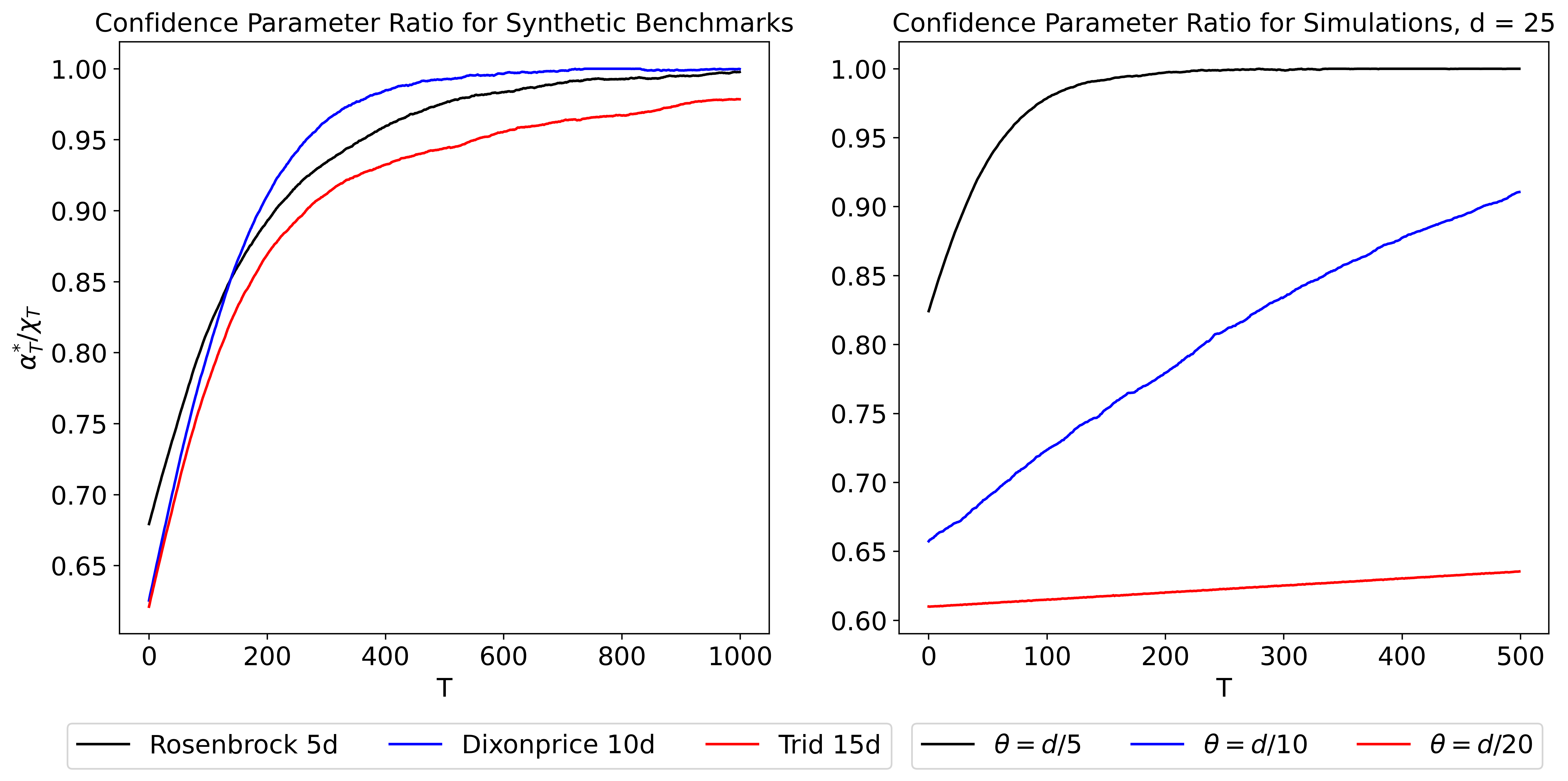}
    \caption{Mean ratio between maximum confidence parameter selected by CSGP-UCB up to time $T$ $\alpha_T^{*}$, and default GP-UCB time $T$ confidence parameter $\chi_T$. }
    \label{fig:alpha_t}
\end{figure}

This section studies the behavior of the ratio $\alpha_T^{*}/\chi_T$, which is the factor by which conditioning on concavity tightens the CSGP-UCB confidence bound (and hence its regret bound) relative to standard GP-UCB with no concavity information. Smaller values indicate that concavity imposes strong truncation on the posterior, leading to a substantial reduction in uncertainty and improved function estimation. Values near one indicate a weak truncation effect, so concavity information provides little additional benefit.

Figure \ref{fig:alpha_t} plots the evolution of $\alpha_T^{*}/\chi_T$, averaged over 25 replications for the benchmarks and numerical simulations with $d=25$. In both the benchmark and simulated experiments, $\alpha_T^{*}/\chi_T$ is well below 1 during the exploration phase when the algorithm is still moving toward the optimum. We see from the regret experiments that this leads to significant reduction in cumulative regret for finite $T$. After exploration behavior reduces, the ratio approaches 1, suggesting that concavity becomes less informative. 
We see that for the most difficult simulation setting, $\theta = d/20$,  $\alpha_T^{*}/\chi_T \ll 1$ throughout the experiment which reflects how concavity conditioning significantly contributes to CSGP-UCB's superior performance relative to all the baseline comparisons.  Overall, these results indicate that concavity constraints can materially accelerate learning in the moderate-$T$ regime that is often most relevant in medical and business settings.

\subsection{Warfarin Dosing Test Function}

\begin{figure}
  \centering
  \includegraphics[width=.8\linewidth]{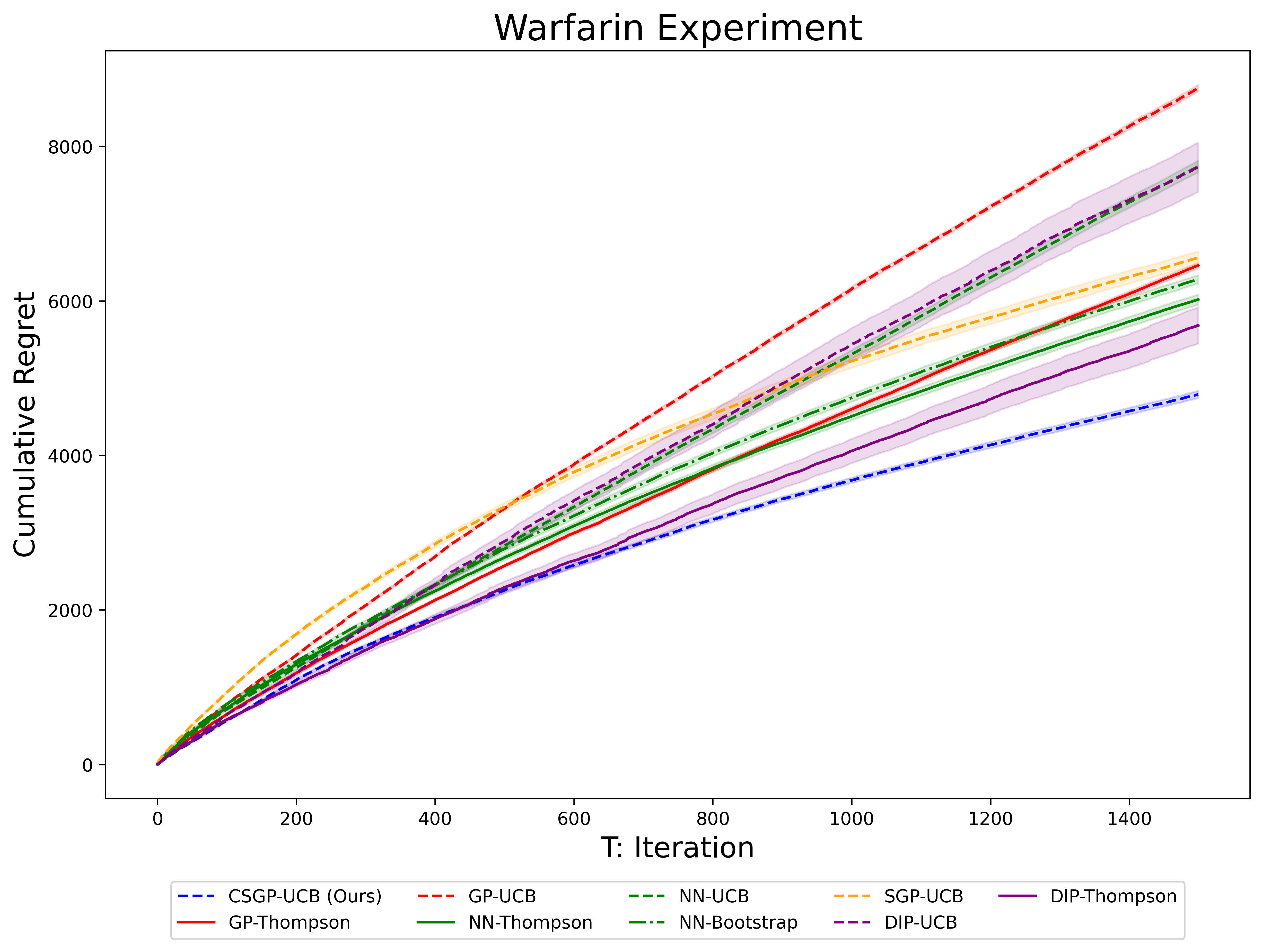}
  \caption{Cumulative regret for the Warfarin test function experiment with $\pm$ 1 SE bars.  
  }
  \label{fig:warfarin1_regret}
\end{figure}

\label{sec:warfarin}
In this section, we examine a test function used by 
\citet{chen2016personalized} to study optimal strategies for personalized 
dosing of Warfarin, an anti-blood clotting medication. Dose optimization is a natural application of concavity assumptions. \cite{calabrese2001u} review the prevalence of the concavity condition in other medical applications. Patient response monotonically improves as dosage approaches the optimal amount from below, and will
 monotonically deteriorate afterwards due to excessive blood thinning \citep{international2009estimation}. Outcomes also depend on patient context variables such as weight, age, and other patient biometrics. The test function's reward measures the difference between the observed value of a patient's biometric values and the optimal value. The function has 7 context dimensions representing patient biometrics/demographics and the action variable is the dosage. Discussion of the application can be found in Appendix \ref{sup:warfarin_details}. 

Figure \ref{fig:warfarin1_regret} displays the mean cumulative regret 
after running the algorithms for $T = 1500$ rounds across $25$ 
replications along with 1 SE bars. We see that CSGP-UCB performs 
best, followed  by DIP-Thompson. Again, we see that 
CSGP outperforms the unconstrained SGP model.

\section{Discussion}
We proposed a contextual bandit algorithm for settings in which
the mean reward is concave in the action given context information. 
We encoded concavity information into a novel spline-based 
GP model via truncation of the posterior 
distribution. 
Novel concentration properties of the resulting truncated posterior were used
to construct the first UCB-algorithm with regret guarantees under Bayesian assumptions and concavity information. In simulation
experiments and an
illustrative biomedical application our algorithm significantly outperformed state-of-the-art non-linear bandit algorithms when the concavity conditions hold. 

There are interesting directions for future work. 
It is an open question whether one can obtain regret guarantees under the 
Frequentist assumptions that the function belongs to some RKHS rather 
than being a draw from a GP prior. Such a bound would allow theoretical 
examination of the advantage of conditioning on concavity information. 
It would also be interesting to scale our algorithm using methods such 
as sparse GPs \citep{titsias2009variational, 
hensman2013gaussian, vakili2021information}.

\begin{acknowledgements} 
    We thank Amazon for funding this research. We also thank Piotr Suder, Joseph Lawson, and Yen-Chun Liu from the Reinforcement Learning Group at Amazon for some extremely helpful discussions. 
\end{acknowledgements}

\bibliography{uai2026-template}

@inproceedings{bouneffouf2020survey,
  title={Survey on applications of multi-armed and contextual bandits},
  author={Bouneffouf, Djallel and Rish, Irina and Aggarwal, Charu},
  booktitle={2020 IEEE Congress on Evolutionary Computation (CEC)},
  pages={1--8},
  year={2020},
  organization={IEEE}
}

@article{silva2022multi,
  title={Multi-armed bandits in recommendation systems: A survey of the state-of-the-art and future directions},
  author={Silva, N{\'\i}collas and Werneck, Heitor and Silva, Thiago and Pereira, Adriano CM and Rocha, Leonardo},
  journal={Expert Systems with Applications},
  volume={197},
  pages={116669},
  year={2022},
  publisher={Elsevier}
}

@article{manjunath2021moments,
  title={Moments calculation for the doubly truncated multivariate normal density},
  author={Manjunath, BG and Wilhelm, Stefan},
  journal={Journal of Behavioral Data Science},
  volume={1},
  number={1},
  pages={17--33},
  year={2021}
}

@article{kan2017moments,
  title={On moments of folded and truncated multivariate normal distributions},
  author={Kan, Raymond and Robotti, Cesare},
  journal={Journal of Computational and Graphical Statistics},
  volume={26},
  number={4},
  pages={930--934},
  year={2017},
  publisher={Taylor \& Francis}
}

@article{russo2014learning,
  title={Learning to optimize via posterior sampling},
  author={Russo, Daniel and Van Roy, Benjamin},
  journal={Mathematics of Operations Research},
  volume={39},
  number={4},
  pages={1221--1243},
  year={2014},
  publisher={INFORMS}
}

@article{ramsay1988monotone,
  title={Monotone regression splines in action},
  author={Ramsay, James O},
  journal={Statistical science},
  pages={425--441},
  year={1988},
  publisher={JSTOR}
}

@article{chen2016personalized,
  title={Personalized dose finding using outcome weighted learning},
  author={Chen, Guanhua and Zeng, Donglin and Kosorok, Michael R},
  journal={Journal of the American Statistical Association},
  volume={111},
  number={516},
  pages={1509--1521},
  year={2016},
  publisher={Taylor \& Francis}
}

@article{valko2013finite,
  title={Finite-time analysis of kernelised contextual bandits},
  author={Valko, Michal and Korda, Nathaniel and Munos, R{\'e}mi and Flaounas, Ilias and Cristianini, Nelo},
  journal={arXiv preprint arXiv:1309.6869},
  year={2013}
}

@inproceedings{
zhang2021neural,
title={Neural Thompson Sampling},
author={Weitong Zhang and Dongruo Zhou and Lihong Li and Quanquan Gu},
booktitle={International Conference on Learning Representations},
year={2021},
url={https://openreview.net/forum?id=tkAtoZkcUnm}
}

@inproceedings{chu2011contextual,
  title={Contextual bandits with linear payoff functions},
  author={Chu, Wei and Li, Lihong and Reyzin, Lev and Schapire, Robert},
  booktitle={Proceedings of the Fourteenth International Conference on Artificial Intelligence and Statistics},
  pages={208--214},
  year={2011},
  organization={JMLR Workshop and Conference Proceedings}
}

@inproceedings{zhou2020neural,
  title={Neural contextual bandits with ucb-based exploration},
  author={Zhou, Dongruo and Li, Lihong and Gu, Quanquan},
  booktitle={International Conference on Machine Learning},
  pages={11492--11502},
  year={2020},
  organization={PMLR}
}

@article{JMLR:v14:hannah13a,
  author  = {Lauren A. Hannah and David B. Dunson},
  title   = {Multivariate Convex Regression with Adaptive Partitioning},
  journal = {Journal of Machine Learning Research},
  year    = {2013},
  volume  = {14},
  number  = {102},
  pages   = {3261--3294},
  url     = {http://jmlr.org/papers/v14/hannah13a.html}
}

@article{mazumder2019computational,
  title={A computational framework for multivariate convex regression and its variants},
  author={Mazumder, Rahul and Choudhury, Arkopal and Iyengar, Garud and Sen, Bodhisattva},
  journal={Journal of the American Statistical Association},
  volume={114},
  number={525},
  pages={318--331},
  year={2019},
  publisher={Taylor \& Francis}
}

@article{wang2016estimating,
  title={Estimating shape constrained functions using Gaussian processes},
  author={Wang, Xiaojing and Berger, James O},
  journal={SIAM/ASA Journal on Uncertainty Quantification},
  volume={4},
  number={1},
  pages={1--25},
  year={2016},
  publisher={SIAM}
}

@article{aubin2020hard,
  title={Hard shape-constrained kernel machines},
  author={Aubin-Frankowski, Pierre-Cyril and Szab{\'o}, Zolt{\'a}n},
  journal={Advances in Neural Information Processing Systems},
  volume={33},
  pages={384--395},
  year={2020}
}

@article{srinivas2009gaussian,
  title={Gaussian process optimization in the bandit setting: No regret and experimental design},
  author={Srinivas, Niranjan and Krause, Andreas and Kakade, Sham M and Seeger, Matthias},
  journal={arXiv preprint arXiv:0912.3995},
  year={2009}
}

@article{prekopa1971logarithmic,
  title={Logarithmic concave measures with applications to stochastic programming},
  author={Pr{\'e}kopa, Andr{\'a}s},
  year={1971},
  journal={ACTA SCIENTIARUM MATHEMATICARUM}
}

@article{krause2011contextual,
  title={Contextual gaussian process bandit optimization},
  author={Krause, Andreas and Ong, Cheng},
  journal={Advances in neural information processing systems},
  volume={24},
  year={2011}
}

@article{varian1982nonparametric,
  title={The nonparametric approach to demand analysis},
  author={Varian, Hal R},
  journal={Econometrica: Journal of the Econometric Society},
  pages={945--973},
  year={1982},
  publisher={JSTOR}
}

@article{boehm2022convex,
  title={Convex supply curves},
  author={Boehm, Christoph E and Pandalai-Nayar, Nitya},
  journal={American Economic Review},
  volume={112},
  number={12},
  pages={3941--3969},
  year={2022},
  publisher={American Economic Association 2014 Broadway, Suite 305, Nashville, TN 37203}
}

@article{meyer2008inference,
  title={Inference using shape-restricted regression splines},
  author={Meyer, Mary C},
  year={2008}
}

@inproceedings{chowdhury2017kernelized,
  title={On kernelized multi-armed bandits},
  author={Chowdhury, Sayak Ray and Gopalan, Aditya},
  booktitle={International Conference on Machine Learning},
  pages={844--853},
  year={2017},
  organization={PMLR}
}

@inproceedings{balandat2020botorch,
  title={{BoTorch: A Framework for Efficient Monte-Carlo Bayesian Optimization}},
  author={Balandat, Maximilian and Karrer, Brian and Jiang, Daniel R. and Daulton, Samuel and Letham, Benjamin and Wilson, Andrew Gordon and Bakshy, Eytan},
  booktitle = {Advances in Neural Information Processing Systems 33},
  year={2020},
  url = {http://arxiv.org/abs/1910.06403}
}

@inproceedings{kassraie2022neural,
  title={Neural contextual bandits without regret},
  author={Kassraie, Parnian and Krause, Andreas},
  booktitle={International Conference on Artificial Intelligence and Statistics},
  pages={240--278},
  year={2022},
  organization={PMLR}
}

@inproceedings{wu2024fast,
  title={A Fast, Robust Elliptical Slice Sampling Method for Truncated Multivariate Normal Distributions},
  author={Wu, Kaiwen and Gardner, Jacob R},
year = {2024},
  booktitle={NeurIPS 2024 Workshop on Bayesian Decision-making and Uncertainty}
}

@article{meyer1968consistent,
  title={The consistent assessment and fairing of preference functions},
  author={Meyer, Richard F and Pratt, John W},
  journal={IEEE Transactions on Systems Science and Cybernetics},
  volume={4},
  number={3},
  pages={270--278},
  year={1968},
  publisher={IEEE}
}

@inproceedings{
riquelme2018deep,
title={Deep Bayesian Bandits Showdown:  An Empirical Comparison of Bayesian Deep Networks for Thompson Sampling},
author={Carlos Riquelme and George Tucker and Jasper Snoek},
booktitle={International Conference on Learning Representations},
year={2018},
url={https://openreview.net/forum?id=SyYe6k-CW},
}

@book{Rasmussen2006Gaussian,
  author = {Rasmussen, Carl Edward and Williams, Christopher K. I.},
  publisher = {The MIT Press},
  title = {Gaussian Processes for Machine Learning},
  year = 2006
}

@inproceedings{vakili2021information,
  title={On information gain and regret bounds in gaussian process bandits},
  author={Vakili, Sattar and Khezeli, Kia and Picheny, Victor},
  booktitle={International Conference on Artificial Intelligence and Statistics},
  pages={82--90},
  year={2021},
  organization={PMLR}
}

@article{chen2012joint,
  title={Joint optimization and variable selection of high-dimensional Gaussian processes},
  author={Chen, Bo and Castro, Rui and Krause, Andreas},
  journal={arXiv preprint arXiv:1206.6396},
  year={2012}
}

@inproceedings{bogunovic2016time,
  title={Time-varying Gaussian process bandit optimization},
  author={Bogunovic, Ilija and Scarlett, Jonathan and Cevher, Volkan},
  booktitle={Artificial Intelligence and Statistics},
  pages={314--323},
  year={2016},
  organization={PMLR}
}

@article{wang2015optimization,
  title={Optimization as estimation with Gaussian processes in bandit settings},
  author={Wang, Zi and Zhou, Bolei and Jegelka, Stefanie},
  journal={arXiv preprint arXiv:1510.06423},
  year={2015}
}

@article{kathuria2016batched,
  title={Batched gaussian process bandit optimization via determinantal point processes},
  author={Kathuria, Tarun and Deshpande, Amit and Kohli, Pushmeet},
  journal={Advances in neural information processing systems},
  volume={29},
  year={2016}
}

@article{hensman2013gaussian,
  title={Gaussian processes for big data},
  author={Hensman, James and Fusi, Nicolo and Lawrence, Neil D},
  journal={arXiv preprint arXiv:1309.6835},
  year={2013}
}

@article{titsias2009variational,
  title={Variational model selection for sparse Gaussian process regression},
  author={Titsias, Michalis K},
  journal={Report, University of Manchester, UK},
  year={2009}
}

@article{calabrese2001u,
  title={U-shaped dose-responses in biology, toxicology, and public health},
  author={Calabrese, Edward J and Baldwin, Linda A},
  journal={Annual review of public health},
  volume={22},
  number={1},
  pages={15--33},
  year={2001},
  publisher={Annual Reviews 4139 El Camino Way, PO Box 10139, Palo Alto, CA 94303-0139, USA}
}

@article{reynolds2010potential,
  title={Potential relevance of bell-shaped and u-shaped dose-responses for the therapeutic targeting of angiogenesis in cancer},
  author={Reynolds, Andrew R},
  journal={Dose-response},
  volume={8},
  number={3},
  pages={dose--response},
  year={2010},
  publisher={SAGE Publications Sage CA: Los Angeles, CA}
}

@article{portier1998u,
  title={U-shaped dose-response curves for carcinogens},
  author={Portier, Christopher J and Ye, Frank},
  journal={Human \& Experimental Toxicology},
  volume={17},
  number={12},
  pages={705--707},
  year={1998},
  publisher={Sage Publications Sage CA: Thousand Oaks, CA}
}

@article{baldi2005inverted,
  title={The inverted “u-shaped” dose-effect relationships in learning and memory: modulation of arousal and consolidation},
  author={Baldi, Elisabetta and Bucherelli, Corrado},
  journal={Nonlinearity in biology, toxicology, medicine},
  volume={3},
  number={1},
  pages={nonlin--003},
  year={2005},
  publisher={SAGE Publications Sage CA: Los Angeles, CA}
}

@article{davis1990u,
  title={U-Shaped dose-response curves: Their occurrence and implications for risk assessment},
  author={Davis, J Michael and Svendsgaard, David J},
  journal={Journal of Toxicology and Environmental Health, Part A Current Issues},
  volume={30},
  number={2},
  pages={71--83},
  year={1990},
  publisher={Taylor \& Francis}
}

@article{carroll1996concavity,
  title={On the concavity of the consumption function},
  author={Carroll, Christopher D and Kimball, Miles S},
  journal={Econometrica: Journal of the Econometric Society},
  pages={981--992},
  year={1996},
  publisher={JSTOR}
}

@article{kandasamy2019multi,
  title={Multi-fidelity gaussian process bandit optimisation},
  author={Kandasamy, Kirthevasan and Dasarathy, Gautam and Oliva, Junier and Schneider, Jeff and Poczos, Barnabas},
  journal={Journal of Artificial Intelligence Research},
  volume={66},
  pages={151--196},
  year={2019}
}

@inproceedings{scarlett2018tight,
  title={Tight regret bounds for Bayesian optimization in one dimension},
  author={Scarlett, Jonathan},
  booktitle={International Conference on Machine Learning},
  pages={4500--4508},
  year={2018},
  organization={PMLR}
}

@inproceedings{da2012gaussian,
  title={Gaussian process modeling with inequality constraints},
  author={Da Veiga, S{\'e}bastien and Marrel, Amandine},
  booktitle={Annales de la Facult{\'e} des sciences de Toulouse: Math{\'e}matiques},
  volume={21},
  number={3},
  pages={529--555},
  year={2012}
}

@article{maatouk2017gaussian,
  title={Gaussian process emulators for computer experiments with inequality constraints},
  author={Maatouk, Hassan and Bay, Xavier},
  journal={Mathematical Geosciences},
  volume={49},
  number={5},
  pages={557--582},
  year={2017},
  publisher={Springer}
}

@article{international2009estimation,
  title={Estimation of the warfarin dose with clinical and pharmacogenetic data},
  author={International Warfarin Pharmacogenetics Consortium},
  journal={New England Journal of Medicine},
  volume={360},
  number={8},
  pages={753--764},
  year={2009},
  publisher={Mass Medical Soc}
}

@book{de1978practical,
  title={A practical guide to splines},
  author={De Boor, Carl and De Boor, Carl},
  volume={27},
  year={1978},
  publisher={springer New York}
}

@inproceedings{durand2018contextual,
  title={Contextual bandits for adapting treatment in a mouse model of de novo carcinogenesis},
  author={Durand, Audrey and Achilleos, Charis and Iacovides, Demetris and Strati, Katerina and Mitsis, Georgios D and Pineau, Joelle},
  booktitle={Machine learning for healthcare conference},
  pages={67--82},
  year={2018},
  organization={PMLR}
}

@article{cereda2021cgptuner,
  title={Cgptuner: a contextual gaussian process bandit approach for the automatic tuning of it configurations under varying workload conditions},
  author={Cereda, Stefano and Valladares, Stefano and Cremonesi, Paolo and Doni, Stefano and others},
  journal={Proceedings of the VLDB Endowment},
  volume={14},
  number={8},
  pages={1401--1413},
  year={2021}
}

@inproceedings{zhou2025gaussian,
  title={A Gaussian process approach for contextual bandit-based battery replacement},
  author={Zhou, Tianshi},
  booktitle={ITM Web of Conferences},
  volume={73},
  pages={01020},
  year={2025},
  organization={EDP Sciences}
}

@misc{hastie2009elements,
  title={The elements of statistical learning},
  author={Hastie, Trevor and Tibshirani, Robert and Friedman, Jerome and others},
  year={2009},
  publisher={Springer series in statistics New-York}
}

@article{kim2012generalization,
  title={A generalization of Caffarelli’s contraction theorem via (reverse) heat flow},
  author={Kim, Young-Heon and Milman, Emanuel},
  journal={Mathematische Annalen},
  volume={354},
  number={3},
  pages={827--862},
  year={2012},
  publisher={Springer}
}

@article{caffarelli2000monotonicity,
  title={Monotonicity properties of optimal transportation and the fkg and related inequalities},
  author={Caffarelli, Luis A},
  journal={Communications in Mathematical Physics},
  volume={214},
  number={3},
  pages={547--563},
  year={2000},
  publisher={Springer}
}

@article{fathi2020proof,
  title={A proof of the Caffarelli contraction theorem via entropic regularization},
  author={Fathi, Max and Gozlan, Nathael and Prod’homme, Maxime},
  journal={Calculus of Variations and Partial Differential Equations},
  volume={59},
  number={3},
  pages={96},
  year={2020},
  publisher={Springer}
}

@incollection{ledoux2006isoperimetry,
  title={Isoperimetry and Gaussian analysis},
  author={Ledoux, Michel},
  booktitle={Lectures on Probability Theory and Statistics: Ecole d’Et{\'e} de Probabilit{\'e}s de Saint-Flour XXIV—1994},
  pages={165--294},
  year={2006},
  publisher={Springer}
}

@article{trinh2015bivariate,
  title={Bivariate conditioning approximations for multivariate normal probabilities},
  author={Trinh, Giang and Genz, Alan},
  journal={Statistics and Computing},
  volume={25},
  number={5},
  pages={989--996},
  year={2015},
  publisher={Springer}
}

@incollection{genz2016numerical,
  title={Numerical computation of multivariate normal probabilities using bivariate conditioning},
  author={Genz, Alan and Trinh, Giang},
  booktitle={Monte Carlo and Quasi-Monte Carlo Methods: MCQMC, Leuven, Belgium, April 2014},
  pages={289--302},
  year={2016},
  publisher={Springer}
}

@article{geyer1994asymptotics,
  title={On the asymptotics of constrained M-estimation},
  author={Geyer, Charles J},
  journal={The Annals of statistics},
  pages={1993--2010},
  year={1994},
  publisher={JSTOR}
}

@incollection{tsybakov2008nonparametric,
  title={Nonparametric estimators},
  author={Tsybakov, Alexandre B},
  booktitle={Introduction to Nonparametric Estimation},
  pages={1--76},
  year={2008},
  publisher={Springer}
}

@article{kur2020convex,
  title={Convex regression in multidimensions: Suboptimality of least squares estimators},
  author={Kur, Gil and Gao, Fuchang and Guntuboyina, Adityanand and Sen, Bodhisattva},
  journal={arXiv preprint arXiv:2006.02044},
  year={2020}
}

@article{he2018instability,
  title={On the instability issue of gradient-enhanced Gaussian process emulators for computer experiments},
  author={He, Xu and Chien, Peter},
  journal={SIAM/ASA Journal on Uncertainty Quantification},
  volume={6},
  number={2},
  pages={627--644},
  year={2018},
  publisher={SIAM}
}

@misc{simulationlib, 
 author = {Surjanovic, S. and Bingham, D.}, 
 title = {Virtual Library of Simulation Experiments: Test Functions and Datasets}, 
 howpublished = {Retrieved February 15, 2026, from \url{http://www.sfu.ca/~ssurjano}} 
}

\newpage

\onecolumn

\title{Exploiting Concavity Information in Gaussian Process Contextual Bandit Optimization\\(Supplementary Material)}
\maketitle

\section{Proof of $\gamma_T$ bound}

\begin{proof}

Recall that the covariance function on the GP induced on our function $f(a, \bm{x})$ takes form:
\begin{align}
&k_f\left( \left(a, \bm{x} \right), \left(a', \bm{x}' \right) \right) = \sum_{j= 1}^J \bs{\phi}_j(a) k_j(\bm{x}, \bm{x}') \bs{\phi}_j(a') \notag .
\end{align}

This is an additive composition of product kernels. The product kernels are product of a Gaussian kernel defined on contexts with a linear kernel defined on actions. We can use the following identity from \citet{krause2011contextual} to derive the maximum information gain $\gamma_T$ for the composition of kernels $k_S, k_Z$ defined on $S, Z$ respectively:
\begin{align*}
&\gamma(T, k_S \otimes k_Z, S \times Z) \leq d \gamma(T, k_S, S) + d \log(T) \\
&\gamma(T, k_S \oplus k_Z, S \times Z) \leq \gamma(T, k_S, S) + \gamma(T, k_Z, Z) + 2 \log(T),
\end{align*}
Where $Z$ has at most dimension $d$. The identities hold with $S, Z$ reversed. \citet{srinivas2009gaussian} show that for the $d$ dimensional Gaussian $\gamma_T = \w{O}(\log(T)^d)$ and the $d$ dimensional Linear kernel is $\gamma_T = \w{O}(d\log(T))$. Putting these results together gives us the Lemma and that for a Gaussian Kernel kernel on $k_j$ we have that for $k_f$:
\begin{align*}
\gamma(T, k_f, \times \w{A} \times \w{X}) = \w{O}(J \log(T)^{d + 1}),
\end{align*}

where $\w{X} \subset \w{R}^d$. 
\end{proof}
\section{Proof of Posterior Form for $\bs{\beta}_t$}
\label{proof:beta_posterior}
\truncposterior*
\begin{proof}
We can write the posterior as 

\begin{align*}
p(\bs{\beta}_t| \bm{y}_{t-1}, \bs{\beta}_t \leq \nu) &= \frac{p(\bm{y}_{t-1}|\bs{\beta}_t)p(\bs{\beta}_t| \bs{\beta}_t \leq \nu) }{C_1} \\
&= \frac{p(\bm{y}_{t-1}|\bs{\beta}_t)p(\bs{\beta}_t) \mathbbm{1}(\bs{\beta}_t \leq \nu) }{C_1 C_2} \\
&\propto p(\bm{y}_{t-1}|\bs{\beta}_t)p(\bs{\beta}_t) \mathbbm{1}(\bs{\beta}_t \leq \nu) .
\end{align*}

Where $C_1 = \int_{-\bs{\infty}}^{\bs{\nu}} p(\bm{y}_{t-1}|\bs{\beta}_t)p(\bs{\beta}_t| \bs{\beta}_t \leq \nu) d \bs{\beta}_t$ and $C_2 = \int_{-\bs{\infty}}^{\bs{\nu}} p(\bs{\beta}_t) d \bs{\beta}_t$. Given our observational model $p(\bm{y}_{t-1}|\bs{\beta}_t)p(\bs{\beta}_t)$ is proportional to Multivariate Normal distribution (Both likelihood and prior are Multivariate Densities), giving us that the posterior in question is a truncated Multivariate Normal.
\end{proof}

\section{Proofs for LCTMVN Concentration Properties}
\label{sup:conc_proofs}

\subsection{Key results from the literature}

We state relevant results from the literature.

\begin{proposition}[Gaussian isoperimetric inequality \citep{ledoux2006isoperimetry}]
\label{prop:isopermetric}

Let $\gamma_d$ be the standard Gaussian measure on $\mathbb{R}^d$. For a Borel set $A \subseteq \mathbb{R}^d$ and $r \geq 0$ define $A^{(r)}$ to be the Minkowski sum of $A$ and the $d$-dimensional ball with radius $r$. If $A \subseteq \mathbb{R}^d$ is Borel and $H$ is any half-space with $\gamma_d(A) \geq \gamma_d(H)$, then for all $r \geq 0$ we have:

\[
\gamma_d(A^{(r)}) \geq \gamma_d(H^{(r)}).
\]
Equivalently, if $\Phi$ is the standard 1d standard normal CDF and $\gamma_d(A) \geq \Phi(a)$ then
\[
\gamma_d(A^{(r)}) \geq \Phi(a + r).
\]
In particular,
\[
\Phi^{-1}(\gamma_d(A^{r})) \geq \Phi^{-1}(\gamma_d(A)) + r
\]

\end{proposition}

\begin{proposition}[Generalized Caffarelli Contraction Theorem \cite{ kim2012generalization, fathi2020proof}]
\label{prop:cafarelli}
Let $\mu = \exp\{-Q(\bm{x})\}d\bm{x}$ and $\nu = \exp\{-Q(\bm{x}) - V(\bm{x})\}dx$ denote two borel probability measures on Euclidean space $(\mathbb{R}^d, |\cdot|)$, where Q denotes a quadratic function, i.e.:
\[
Q(\bm{x}) = \langle A\bm{x}, \bm{x} \rangle + \langle \bm{b}, \bm{x} \rangle + c
\]
with A positive-definite, and V is a convex function that can possibly take value $+\infty$. Then the Brenier optimal-transport map $T = T_{opt}$ pushing forward $\mu$ onto $\nu$ is a contraction:
\[
\forall \bm{x}, \bm{y} \in R^d \quad ||T(\bm{x}) - T(\bm{y})|| \leq ||\bm{x}- \bm{y}||.
\]

\end{proposition}

\subsection{Proof of Truncation-Insensitive Concentration Theorem Around Median}

\begin{lemma}
\label{lemma:transport}
Define $\bm{z} \sim N(0, \bm{I}_d)$. Define convex region $K$ and random variable $\bm{z}'$ such that $\bm{z}' \overset{d}{=} \bm{z} | \bm{z} \in K$. There exists a contraction map $T$ such that $T(\bm{z}) \overset{d}{=} \bm{z}'$
\end{lemma}

\textit{Proof}: 
The conditional density of $\bm{z}
| \bm{z} \in K$  is proportional to:
\begin{align*}
p(\bm{z} | \bm{z} \in K ) &\propto \exp(-\frac{1}{2}\bm{z}^2) \mathbbm{1}(\bm{z} \in K) \\
              & = \exp\left\{- \left( \frac{1}{2}\bm{z}^2  + V_K(\bm{z}) \right) \right\}, 
\end{align*}

where we define $V_K(\bm{z})$ as 
\[
V_K(\bm{z}) =
\begin{cases}
0, & \text{if $\bm{z} \in K$},\\
\infty, & \text{else}.
\end{cases}
\]

Note that $V_K(\bs{\beta})$ is a convex function as $K$ is convex. Let $Q(\bm{z}) \propto -\frac{1}{2}\bm{z}^2$ and $V(\bm{z}) = V_K(\bm{z})$ be defined above. Then we can apply Proposition \ref{prop:cafarelli} to prove the existence of a contraction map $T$ that pushes forward the base Gaussian measure to the Gaussian measure conditioned on the truncation $K$. In other words, there exists T such that $T(\bm{z}) \overset{d}{=} z'$. \qed

\begin{lemma}
    Let $f(\cdot)$ be a L-Lipshitz function and $\bm{z} \sim \gamma_d$ where $\gamma_d$ is a standard Gaussian measure in $\mathbbm{R}^d$. Let $A = \{\bm{z} \in \mathbbm{R}^d : f(\bm{z}) \leq m\}$ where $m$ is the median of $f(\bm{z})$. Fix $t \geq 0$ and $r = t/L$, then we have 
    \[A^{(r)} \subseteq \{\bm{z}: f(\bm{z}) \leq m + t \}\]
\end{lemma}
\label{lemma:inclusion}

\textit{Proof:} By definition of $A^{(r)}$, for any $\bm{z} \in A^{(r)}$ there exists $\bm{a} \in A$ such that $||\bm{z} - \bm{a}|| \leq r$. Because $f(\cdot)$ is $L$-Lipshitz:
\[
f(\bm{z}) - f(\bm{a}) \leq L||\bm{z} - \bm{a}|| \implies f(\bm{z})  \leq f(\bm{a}) + L||\bm{z} - \bm{a}|| \leq m + Lr = m + t,
\]
which implies the inclusion. \qed

\medCoProp*

\textit{Proof}: 
Define $A = \Sigma^{\frac{1}{2}}$ and let $\bm{z} \sim N(\bm{0}, \bm{I}_d)$.  We can write $\bs{\beta} \overset{d}{=} \mu + A\bm{z}'$, where $\bm{z}' \overset{d}{=} \bm{z} \ | \  \bm{z} \leq A^{-1}(\bs{\nu} - \bs{\mu}) $

Define the truncation set $K = \{\bm{z} \in \mathbb{R}^d: (\bs{\mu} + A\bm{z})_i \leq \bs{\nu}_i, i = 1, \cdots , d \}$. $K$ is a convex regions because it is a affine transformation of the original half-space truncation region. By Lemma \ref{lemma:transport} there exists a transport map $T$ such that $T(\bm{z}) \overset{d}{=} \bm{z}'$.

Now define function $h(\bm{z}) = \bm{c}^T\left\{\bs{\mu} + A T(\bm{z})\right\}$. We can show that for constants $\bm{z}_1, \bm{z_2} \in \mathbbm{R}^d$ that 

\[
||h(\bm{z}_1)  - h(\bm{z}_2)|| \leq ||A^T\bm{c}|| \ ||T(\bm{z}_1) - T(\bm{z}_2)|| \leq \sqrt{\bm{c}^T \Sigma \bm{c}} \ ||\bm{z}_1 - \bm{z}_2||, 
\]

Where we use that $T(\cdot)$ is a contraction, $||A^T\bm{c}|| = \sqrt{\bm{c}^T AA^T \bm{c}}$ and $A = \bs{\Sigma}^{\frac{1}{2}}$. This shows that $h(\bm{z})$ has Lipshitz property $L = \sqrt{\bm{c}^T\Sigma \bm{c}}$.

Define the set $A := \{\bm{z} \in \mathbbm{R}^d: h(\bm{z}) \leq m \}$. By definition $\gamma_d(A) \geq \frac{1}{2}$. Fix $t \geq 0$ and set $r = \frac{b}{\sqrt{\bm{c}^T \Sigma \bm{c}}}$. By using lemma $\ref{lemma:inclusion}$ and taking complements, we have 
\[
\{\bm{z} \in \mathbb{R}^d: h(\bm{z}) > m + t \} \subseteq (A^{(r)})^C \implies \pr{h(\bm{z}) > m + b} 
 \leq 1 - \gamma_d(A^{(r)})
\]
We can now apply the Isoperimetric inequality from Proposition \ref{prop:isopermetric} to $\gamma_d(A^{(r)}) $:
\[
\gamma_d(A) \geq \Phi(0) \implies \gamma_d(A^{(r)}) \geq \Phi(0 + r)
\]
Thus we arrive at the one tail inequality
\[
\pr{h(\bm{z})  > b + m} = \pr{\bs{\beta} -m  > t} \leq 1 - \Phi\left\{\frac{b}{\sqrt{\bm{c}^T \Sigma \bm{c}}} \right\}
\]
To bound left tail, $\pr{\bs{\beta}  \leq m-c}$ note that $-m$ is the median of $-\bs{\beta}$, and that $-h(\bm{z})$ has the same Lipchitz constant as $h(\bm{z})$. Therefore we can apply the same argument above and obtain, 
\[
\pr{\bs{\beta} \leq m- c} = \pr{-h(\bs{\beta}) \geq -m + b} \leq 1 - \Phi\left\{\frac{b}{\sqrt{\bm{c}^T \Sigma \bm{c}}} \right\}.
\]

The union bound gives us the final two tailed bound. \qed

\subsection{Proofs of Truncation Sensitive Concentration Around Mean }

\begin{lemma}
\label{lemma:log_concave}
Let $C$ be a convex region and $p_{\bs{\mu}, \Sigma}(\bs{\beta})$ be the multivariate normal density with mean $\bs{\mu}$ and covariance $\Sigma$. Then $q(\bs{\mu}) = \int_C p_{\bs{\mu}, \Sigma}(\bs{\beta}) d \bs{\beta}$ is log concave in $\bs{\mu}$. Equivalently $-log\{q(\bs{\mu})\}$ is convex.
\end{lemma}

\textit{Proof}: We can rewrite the multivariate normal density as 

\[
p_{\bs{\mu}, \Sigma}(\bs{\beta}) = p_{\bm{0}, \Sigma}(\bs{\beta} - \bs{\mu}), 
\]
which shows that it is log-concave with respect to $\bs{\beta}$ and $\bs{\mu}$. By Prekopa's theorem, integration over a convex region preserves log concavity \citep{prekopa1971logarithmic}. Therefore $q(\bs{\mu})$ is log-concave.

\begin{lemma}
\label{lemma:mean_shift}
Suppose $\bs{\beta} \sim N_{\bs{\nu}}(\bs{\mu}, \Sigma)$ and let $p_{\bs{\mu}, \Sigma}(\bs{\beta})$ be the standard multivariate normal density with mean $\bs{\mu}$ and covariance $\Sigma$. Define 
\[
q(\bs{\mu}) = \int_K p_{\bs{\mu, \Sigma}}(\bs{\beta}) d\bs{\beta}
\]
Where $K = \{\bs{\beta} \in \mathbb{R}^d: \bs{\beta}_i \leq \bs{\nu}_i, i, \cdots, d \}$. Let $\bs{\mu}^{*} = \mathbb{E}
(\bs{\beta})$, then: 
\[
\bs{\mu}^{*} = \bs{\mu} + \Sigma \nabla g(\bs{\mu}),
\]
where $g(\bs{\mu}) =  \log\{ q(\bs{\mu}) \}$.
\end{lemma}
\textit{Proof:}
We can differentiate under the integral of $q(\bs{\mu})$ as the density is gaussian:

\[
\nabla_{\bs{\mu}} q(\bs{\mu}) = \int_K \nabla_{\bs{\mu}} p_{\bs{\mu}, \bs{\Sigma}}(\bs{\beta}) d\bs{\beta} = \int_K \Sigma^{-1} (\bs{\beta} - \bs{\mu}) p_{\bs{\mu}, \Sigma}(\bs{\beta}) d\bm{\bs{\beta}}= \Sigma^{-1}\left[ \mathbb{E}\{\bm{z}\mathbbm{1}_{K}(\bs{\beta)} \} - \bs{\mu}q(\bs{\mu})\right].
\]

From this we can divide both of the expression by $q(\bs{\mu})$ to obtain:
\[
\frac{\nabla q(\bs{\mu})}{q(\bs{\mu})} = \Sigma^{-1}(\bs{\mu}^{*} - \bs{\mu}).
\]
Noting that $\frac{\nabla q(\bs{\mu})}{q(\bs{\mu})} =  \nabla g(\bs{\mu})$ we  arrive at the result:
\[
\bs{\mu}^{*} = \bs{\mu} + \Sigma \nabla g(\bs{\mu}).
\] \qed

\begin{lemma}
\label{lemma:mgf}
    Suppose $\bs{\beta} \sim N_{\bs{\nu}}(\bs{\mu}, \Sigma)$ and define $\bs{\mu}^{*} = \mathbb{E}(\bs{\beta})$. Let $\bm{c}$ be a vector of constants in $\mathbb{R}^d$ and define $u = \bm{c}^T(\bs{\beta} - \bs{\mu}^{*})$. Then the moment generating function of $u$, $\psi(k)$, is:
    \[
    \psi(k) := \mathbb{E}\{\exp(ku) \} = \exp\left(\frac{1}{2} k^2 \sigma^2 \right) \exp\left\{ g(\bs{\mu} + k\Sigma \bm{c}) - g(\bs{\mu}) - k \bm{c}^T(\bs{\mu}^{*} - \bs{\mu}) \right\},
    \]
where $\sigma^2 = \bm{c}^t \Sigma \bm{c}$ and $g(\bs{\mu})$ is defined as in Lemma \ref{lemma:mean_shift}
\end{lemma}

\textit{Proof}: 
\[
\mathbb{E}\{\exp(ku) \} = \exp(-k\bm{c}^T \bs{\mu}^{*}) \frac{\mathbbm{E} \left\{\exp(k\bm{c}^T \bs{\beta}) \mathbbm{1}(\bs{\beta} \preccurlyeq \bs{\nu}) \right\}}{q(\bs{\mu})}.
\]
Completing the square we can derive using standard gaussian identities that:
\
\[
\exp(k\bm{c}^T \bs{\beta}) p_{\bs{\mu}, \Sigma}(\bs{\beta}) = \exp\left(k \bm{c}^T \bs{\mu} + \frac{1}{2}k^2 \sigma^2 \right) p_{\bs{\mu} + k \Sigma \bm{c}, \Sigma}(\bs{\beta}).
\]
Therefore
\[
\mathbb{E}\left\{ \exp(k \bm{c}^T \bs{\beta}) \mathbbm{1}(\bs{\beta} \preccurlyeq \bs{\nu}) \right\} =  \exp\left(k \bm{c}^T \bs{\mu} + \frac{1}{2}k^2 \sigma^2 \right) q(\bs{\mu} + k \Sigma \bm{c})
\]

Plugging this into back into the original expression we get our result
\[
\mathbb{E}\{\exp(ku) \} = \exp\left(\frac{1}{2} k^2 \sigma^2 \right)\exp\left\{g(\bs{\mu} + k \Sigma \bm{c}) - g(\bs{\mu}) - k\bm{c}^T (\bs{\mu}^{*} - \bs{\mu})  \right\}
\] \qed

\meanConcBregman*
\textit{Proof:} We drop the subscripts for clarity. We first examine the one sided bound, $\pr{ \bm{c}^T(\bs{\beta} - \bs{\mu}^{*} \geq b}$. We apply the Chernoff bound using the moment generating function derived in Lemma \ref{lemma:mgf}
with the optimal $k = \frac{b}{\sigma^2}$:
\begin{align*}
\pr{ \bm{c}^T(\bs{\beta} - \bs{\mu}^{*}) \geq b} &\leq \inf_k \exp(-k b) \exp\left(\frac{1}{2} k^2 \sigma^2 \right)\exp\left\{g(\bs{\mu} + k \Sigma \bm{c}) - g(\bs{\mu}) - k\bm{c}^T (\bs{\mu}^{*} - \bs{\mu})  \right\} \\
& = \exp\left( -\frac{b^2}{2\sigma^2} \right)\frac{q\left(\bs{\mu} + \frac{b}{\sigma^2} \Sigma \bm{c}\right)}{q(\bs{\mu})}\exp\left(-\frac{b}{\sigma^2} d \right) 
\end{align*}

We can obtain the other tail by applying the Chernoff bound to $\bm{c}^T(\bs{\beta} - \bs{\mu}^{*})$ and taking the union bound with the above bound. This gives the first line of the result. 

To relate the inequality to Bregman divergences recall from Lemma \ref{lemma:mean_shift} that $\bs{\mu}^{*} = \bs{\mu} + \Sigma \nabla g(\bs{\mu})$. Using this we can write:
\[
\bm{c}^T(\bs{\mu}^{*} - \bs{\mu}) = \langle \nabla g(\bs{\mu}), \Sigma \bm{c} \rangle
\]
So that we can rewrite factors in terms of $G$
\begin{align*}
-\log \left\{ \frac{q\left(\bs{\mu} + \frac{b}{\sigma^2} \Sigma \bm{c}\right)}{q(\bs{\mu})}\exp\left(-\frac{b}{\sigma^2} d \right) \right\} = G(\bs{\mu} + \frac{b}{\sigma^2} \Sigma \bm{c}) - G(\bs{\mu}) - \langle \nabla G(\bs{\mu}), \frac{b}{\sigma^2} \Sigma \bm{c} \rangle.
\end{align*}

Notice that by Lemma \ref{lemma:log_concave}, $G(\theta)$ is strictly convex and this value is greater than zero. Therefore it is a valid Bregman divergence and:
\[
G\left(\bs{\mu} + \frac{b}{\sigma^2} \Sigma \bm{c}\right) - G(\bs{\mu}) - \langle \nabla G(\bs{\mu}), \frac{b}{\sigma^2} \Sigma \bm{c} \rangle = D_{G, \bm{c}}^{\bs{\mu}, \bs{\Sigma}}\left(\bs{\mu} + \frac{b}{\sigma^2} \Sigma \bm{c}, \bs{\mu} \right).
\]

The left tail follows the exact same argument. Plugging this form into the bound provides the result. \qed

\begin{lemma}
\label{lemma:monotonicity}
Define function  
\begin{align*}
h(b) := \exp\left[ -D_{G, \bm{c}}^{\bs{\mu}, \bs{\Sigma}}\left(\bs{\mu} + \frac{b}{\sigma^2} \Sigma \bm{c}, \bs{\mu}\right) \right] + \exp\left[ -D_{G, \bm{c}}^{\bs{\mu}, \bs{\Sigma}}\left(\bs{\mu} - \frac{b}{\sigma^2} \Sigma \bm{c}, \bs{\mu}\right) \right]
\end{align*}

$h(b)$ is decreasing for $b \geq 0$. 
\end{lemma}

\textit{Proof}: Define $\bm{v} = \frac{1}{\sigma^2} \Sigma \bm{c}$ and write:
\begin{align*}
k(b) = D_{G, \bm{c}}^{\bs{\mu}, \bs{\Sigma}}\left(\bs{\mu} + \frac{b}{\sigma^2} \Sigma \bm{c}, \bs{\mu}\right) = G(\bs{\mu} + b\bm{v}) - G(\bs{\mu}) - b \langle \nabla G(\bs{\mu}), \bm{v} \rangle
\end{align*}
Differentiating $k$ with respect to $b$ we obtain:
\begin{align*}
k'(b) = \langle \nabla G(\bs{\mu}  + b\bm{v}) - \nabla G(\bs{\mu}), \bm{v}\rangle.
\end{align*}
Given the strict-convexity of $G$, we know that its gradients are monotone. Therefore:
\begin{align*}
\langle \nabla G(\bs{\mu} + b\bm{v}) - \nabla G(\bs{\mu}), b\bm{v} \rangle \geq 0 \implies k'(b) \geq 0.
\end{align*}

Therefore we see that $k(b)$ is increasing. The exact same argument applies to $D_{G, \bm{c}}^{\bs{\mu}, \bs{\Sigma}}\left(\bs{\mu} - \frac{b}{\sigma^2} \Sigma \bm{c}, \bs{\mu}\right)$, which shows that it is increasing. Therefore $h(b)$ is decreasing. \qed

\subsection{Combined Bound Concentration Proof}

\begin{lemma}
If $P_t(a, \chi_t) <  1$ and $\chi_t \geq 1$ then $\chi_t^{*} < \chi_t$ exists.
\end{lemma}

\textit{Proof}: From Lemma \ref{lemma:monotonicity} we have that $P_t(a,  x)$ is monotone decreasing with respect to $s$. Therefore $\exp\left(-\frac{s}{2}\right) P_t(a, s \sigma_{\bm{c}_t})$ is monotone decreasing in $s$. We also have by assumption:
\begin{align*}
\exp\left(-\frac{\chi_t}{2}\right) P_t(a,\chi_t \sigma_{\bm{c}_t}(a, \bm{x}_t)) < \exp\left(-\frac{\chi_t}{2}\right)
\end{align*}

All functions are monotone decreasing and continuous. In addition $\exp\left(-\frac{0}{2}\right) P_t(a, 0) = 1 \geq \exp\left(-\frac{\chi_t}{2}\right)$. Therefore by the intermediate value theorem, there exists $s < \chi_t$ such that the inequality is fulfilled. \qed

\combinedbound*
\textit{Proof:} First note that at time $t$, upon observing information according to our assumptions,  $f(a, \bm{x}_t) = \bm{c}_t(a)^T \bs{\beta}_t$ where $\bs{\beta}_t$ is distributed according to the upper truncated normal distribution $N_{\bs{\nu}}(\bs{\mu}_t, \Sigma_t)$. 

First consider the case where $P_t(a,\chi_t) \geq 1$. Then the confidence region bound is centered around the median using $\chi_t$. Applying Proposition \ref{prop:median_concentration} and noting from \cite{srinivas2009gaussian} that $2\left\{ 1 - \Phi\left( \chi_t^{1/2} \right) \right\} \leq \exp\left(-\frac{\chi_t}{2} \right)$ the statement holds. 

Now consider the case where $P_t(a, \bm{x}_t, \chi_t) < 1$. By construction and application of Proposition \ref{prop:mean_concentration} we can write:
\begin{align*}
    \pr{| f(a, \bm{x}_t)  - \mu_t^{*}(a, \bm{x}_t)| \geq \alpha_t(a, \bm{x}_t) \sigma_{\bm{c}_t}} \leq \exp\left(-\frac{\chi_t^{*}}{2}\right)P_t(a,  \chi_t^{*} \sigma_{\bm{c})t}(a, \bm{x}_t)) \leq \exp\left(-\frac{\chi_t}{2} \right).
\end{align*} \qed

\section{Maximum Information Gain}
\label{sec:mig}

Following \citet{srinivas2009gaussian}, we will use information theoretical quantities to bound the regret of our algorithm. Therefore we will be interested in the mutual information $\mbox{I}(\bm{y}_{A}, f)$ between the set of observations $\bm{y}_{A}$ and the underlying function $f$. $\mbox{I}(\bm{y}_{A}, f)$ represents the reduction in uncertainty surrounding the mean reward function $f$ upon observing rewards $\bm{y}_{A}$. Our regret bounds will take the form $\w{O}^{*}(\sqrt{T \gamma_T} )$ 
where $\gamma_T$ is the \textit{maximal information gain} defined as 
\begin{align*}
\gamma_T := \max_{A \subset \w{A}, |A| = T} \mbox{I}(\bm{y}_{A}, f) .
\end{align*}
The maximal information gain, $\gamma_T$, for our model can be bounded using the results of 
\citet{krause2011contextual}. The GP induced on
$f(a, \bm{x})$ by our model has a covariance function composed of $J$ additive components. 
Each additive component is the product of a generic $d$-dimensional covariance function (such 
as Gaussian or Mat\'ern) and a $1$-dimensional linear kernel derived from the C-Spline features. \citet{krause2011contextual} show how to bound $\gamma_T$ for additive and product compositions. Thus, we are able to obtain the following 
asymptotic bounds for the CSGP model. 
\begin{lemma}
  \label{lemma:information_gain}
  Let the prior for the expected reward function $f(a, \bm{x})$ be the CSGP model defined in Equation \ref{eqn:function_model} such that the maximal information gain of each $k_j(\cdot, \cdot)$ at time $T$ is bounded by $\gamma^{*}_T$. Then the maximal information gain for the CSGP covariance function satisfies
  \begin{align*}
  \gamma_T \leq J \gamma^{*}_T + 4J \log(T).
  \end{align*}
\end{lemma}
\citet{vakili2021information} provides bounds for $\gamma^{*}_T$ for common kernels. Throughout this paper we select $k_j(\cdot, \cdot)$ to be Gaussian kernels so that the maximum information gain for CSGP satisfies $\gamma_T = \w{O}\left\lbrace J\log(T)^{d+1}\right\rbrace$. 
\subsection{Proof of Lemma \ref{lemma:information_gain}}

\begin{proof}

Recall that the covariance function on the GP induced on our function $f(a, \bm{x})$ takes form:
\begin{align}
&k_f\left( \left(a, \bm{x} \right), \left(a', \bm{x}' \right) \right) = \sum_{j= 1}^J \bs{\phi}_j(a) k_j(\bm{x}, \bm{x}') \bs{\phi}_j(a') \notag .
\end{align}

This is an additive composition of product kernels. The product kernels are product of a Gaussian kernel defined on contexts with a linear kernel defined on actions. We can use the following identity from \citet{krause2011contextual} to derive the maximum information gain $\gamma_T$ for the composition of kernels $k_S, k_Z$ defined on $S, Z$ respectively:
\begin{align*}
&\gamma(T, k_S \otimes k_Z, S \times Z) \leq d \gamma(T, k_S, S) + d \log(T) \\
&\gamma(T, k_S \oplus k_Z, S \times Z) \leq \gamma(T, k_S, S) + \gamma(T, k_Z, Z) + 2 \log(T),
\end{align*}
Where $Z$ has at most dimension $d$. The identities hold with $S, Z$ reversed. \citet{srinivas2009gaussian} show that for the $d$ dimensional Gaussian $\gamma_T = \w{O}(\log(T)^d)$ and the $d$ dimensional Linear kernel is $\gamma_T = \w{O}(d\log(T))$. Putting these results together gives us the Lemma and that for a Gaussian kernel on $k_j$ we have that for $k_f$:
\begin{align*}
\gamma(T, k_f, \times \w{A} \times \w{X}) = \w{O}(J \log(T)^{d + 1}),
\end{align*}

where $\w{X} \subset \w{R}^d$. 

\end{proof}

\section{Regret Proofs}

Our proof for the UCB regret follows the ideas presented by \cite{srinivas2009gaussian, krause2011contextual}. First we prove that given our choice of $\chi_t$ the our upper confidence bound holds with high probability over the finite decision set over all $t$ and $a \in \w{A}$:

\begin{lemma}
\label{lemma:prob_over}
Pick $\delta \in (0, 1)$. Set $\chi_t = 2\log( |\w{A}|t^2 \pi^2/6 \delta)$. Then
\begin{align*}
\pr{\forall t, \forall a \in \w{A}, | f(a, \bm{x}_t) - \kappa_t(a, \bm{x}_t)| \leq \sigma_{\bm{c}_t}(a_t, \bm{x}_t) \alpha_{t}(a, \bm{x}_t)^{\frac{1}{2}}} \geq 1 - \delta .
\end{align*}
\end{lemma}

\begin{proof}
We apply the Union bound over $a$ in the bound obtained in Lemma \ref{lemma:myucb} so that 
\begin{align*}
\pr{\forall a \in \w{A}, | f(a, \bm{x}_t) - \kappa_t(a, \bm{x}_t)| \leq \sigma_{\bm{c}_t}(a_t, \bm{x}_t) \alpha_{t}(a, \bm{x}_t)^{\frac{1}{2}}} &= \pr{\forall a \in \w{A}, | \bm{c}_t(a)^T \bs{\beta}_t - \kappa_t(a, \bm{x}_t)| \leq \sigma_{\bm{c}_t}(a_t, \bm{x}_t) \alpha_{t}(a, \bm{x}_t)^{\frac{1}{2}}} \\
&\geq 1 - |\w{A}| \exp\left(-\frac{1}{2} \chi_t \right)
\end{align*}
Where the inequality follows from the concentration inequality proved in Lemma \ref{lemma:myucb}. Taking a Union bound with respect to $t$, choosing $\chi_t$ such that $ |\w{A}| \exp\left( -\frac{\chi    _t}{2} \right) = \frac{6\delta}{\pi^2 t^2 }$, then the statement holds. 
\end{proof}

We also prove the following concentration results that will used for the continuous case where there are increasingly fine discretizations $\w{A}_t$.

We first state the following lemma that adapts lemma 5.5 from \citep{srinivas2009gaussian} for our purposes:

\begin{lemma} \label{lemma:prob_over2}
Pick $\delta \in (0, 1)$ and $\alpha_t = 2\log\left( \pi_t/\delta \right)$ where $\sum_{t \geq 1} \pi_t = 1, \pi_t > 0$ then

\begin{align*}
|f(a_t, \bm{x}_t) - \kappa(a_t, \bm{x}_t)| \leq \sigma_{\bm{c}_t}(a_t, \bm{x}_t) \alpha_t^{\frac{1}{2}}(a_t, \bm{x}_t) \quad \forall t \geq 1 , 
\end{align*}
holds with probability $1 - \delta$.
\end{lemma}

\begin{proof}
Fix $t \geq 1$ and $a \in \w{A}, \bm{x} \in \w{X}$. Lemma \ref{lemma:myucb} already showed that conditioning on the action/contexts, rewards, and concavity observed before taking action time $t$ :
\begin{align*}
\pr{| f(a, \bm{x}_t) - \kappa_{t}(a, \bm{x}_{t})| > \sigma_{\bm{c}_t}(a, \bm{x}_{t}) \alpha_{t}(a, \bm{x}_t)^{\frac{1}{2}}} \leq \exp\left(\frac{-\chi_t}{2} \right)
\end{align*}
We have that $\exp\left(-\frac{\chi_t}{2}\right) = \frac{\delta}{\pi_t}$ so we can take the union bound for $t \in \mathbb{N}$ to prove the statement.
\end{proof}

\regretlemma*

\textit{Proof}
Given the inequality proven in lemma \ref{lemma:prob_over}, we can directly apply lemma 4.1 from \cite{krause2011contextual} which gives 
\[
r_t \leq \alpha_t(a, \bm{x}_t)^{\frac{1}{2}} \sigma_{\bm{c}_t}(a, \bm{x}_t).
\]
An argument from \cite{russo2014learning} shows that that $\sum_{t=1}^T \sigma^2_{\bm{c}_t}(a, \bm{x}_t) \leq \gamma_T$. The argument follows.  \qed
\regret*

\begin{proof}
    First we study the discrete case. Conditional on our high probability bound results proved in Lemmas \ref{lemma:prob_over}, \ref{lemma:prob_over2}, the arguments follow \cite{krause2011contextual} up to Theorem 5. Given the arguments up to that point we have with probability  $\geq 1 - \delta$ that
    \[
    r_t^2 \leq 4\alpha_t \sigma^2_t(a, \bm{x}_t)
    \]
However, in our case, the confidence parameter $\alpha_t$ is no longer strictly increasing. Therefore $\alpha_T$ is not necessarily the largest value. So instead we take $\alpha_T^{*} = \max \{\alpha_t\}_{t=1}^T$:
\[
 r_t^2 \leq 4\alpha_t^{*} \sigma^2_t(a, \bm{x}_t)
\]
and the theorem follows from the same arguments as in \cite{krause2011contextual}. Note that $\alpha_T^{*} \leq \chi_T$ by construction. 

The continuous case follows \cite{krause2011contextual} similarly conditional on Lemma \ref{lemma:prob_over2}. We replace the  upper bound confidence parameter $\alpha_T(a, \bm{x}_t)$ with $\alpha_T^{*}$, the maximum over all $\alpha_t(a, \bm{x}_t)$ up to time $T$.  

\end{proof}

\section{Example of LCTMVN Tails Heavier Than Original Gaussian}
\label{sec:conc_counter}
\begin{figure}
    \centering
    \includegraphics[width=0.5\linewidth]{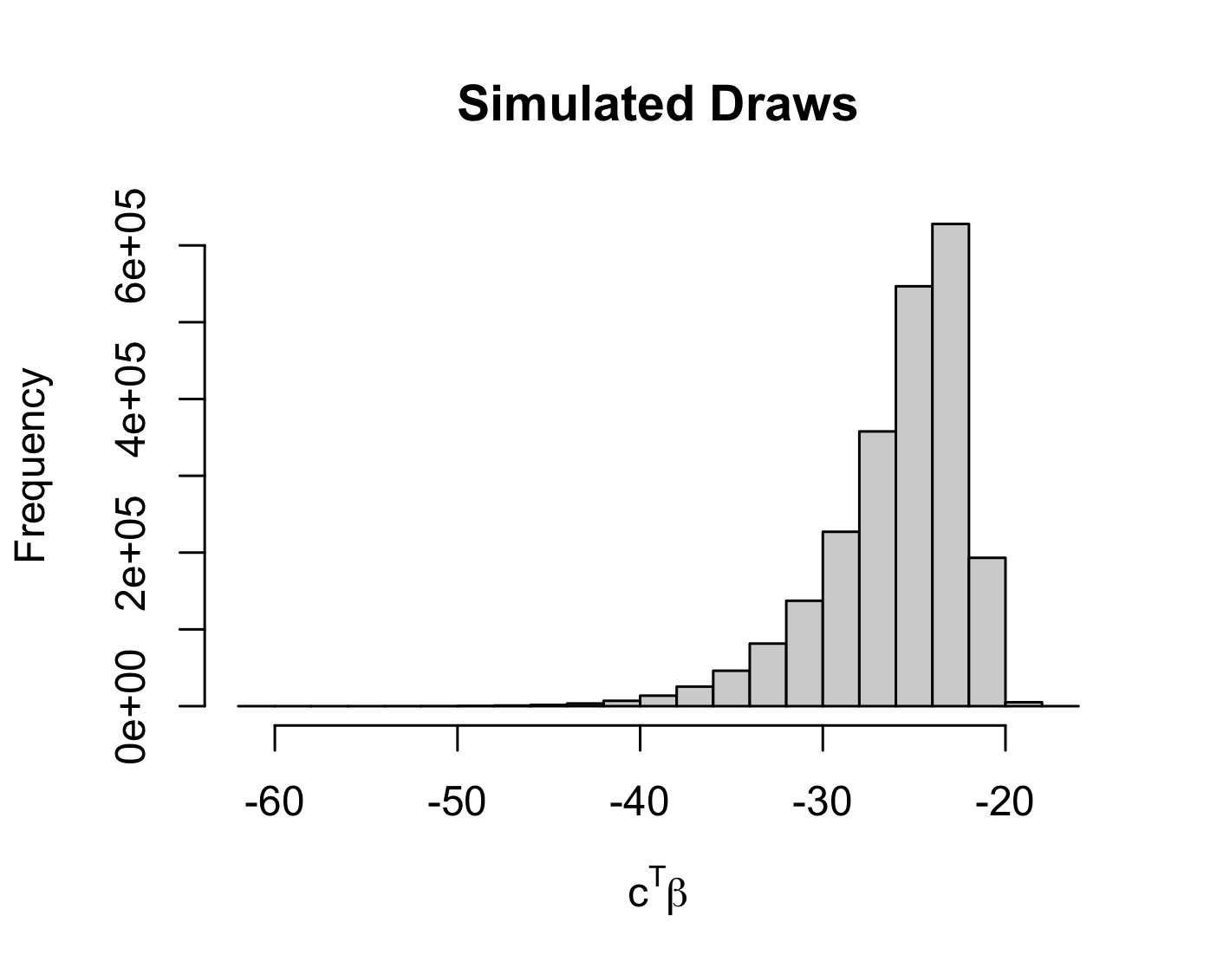}
    \caption{Example draws of the LCTMVN variable. The distribution is very skewed. }
    \label{fig:exmample_draws}
\end{figure}

For exposition, we provide an easily understandable and reproducible LCTMVN produced from a truncated bivariate normal distribution. Take $\bm{c} = (1, 10)^T$ and define the truncation parameter to be $\bs{\nu} = (-2, 10)$. Define a bivariate normal distribution with unit variance and correlation parameter $\rho = .995$. Simulated draws of the the LCTMVN variable of the given MVN distribution with truncation parameter $\bs{\nu}$ are shown in Figure \ref{fig:exmample_draws}. we see that the distribution has a heavy right tail. In fact one can calculate analytically/numerically that this distribution has weaker concentration properties than a standard gaussian. Let $\bs{\beta}'$ be a draw from the un-truncated MVN distribution defined above. Then we have 

\begin{align*}
\pr{\bm{c}^T \bs{\beta} <\ E(\bm{c}^T \bs{\beta}) + 1} \approx .501 \\
\pr{\bm{c}^T \bs{\beta} <\ E(\bm{c}^T \bs{\beta'}) + 1} \approx .536\\
\end{align*}

Which shows that the LCTMVN has exceeds standard Gaussian concentration for this width parameter. In addition, this probability for the LCTMVN violates the standard gaussian bandit bound $\frac{1}2\exp(-1/\sigma^2_c) = .51$. Although, the correlation here is extreme and the confidence tight, we have found empirically that this violation of gaussian concentration can get much worse in common cases for higher dimensions.

\section{UCB Computation}
\label{sup:computation}
Computing our bound requires computing the LCTMVN variable means ($\bs{\mu}^{*}$), medians ($m(a, \bm{x}_t)$), orthant probabilities ($\pr{\bs{\beta}_t \preccurlyeq \bs{\nu}}$), and optimal confidence parameter $\chi_t^{*}$. Luckily there exists a vast literature for truncated multivariate normal computation. The mean can be calculated with high accuracy using identities from \cite{kan2017moments, manjunath2021moments}. Orthant probabilities can be computed efficiently using algorithms and integral approximations provided by \cite{trinh2015bivariate, genz2016numerical}. We calculate posterior medians via efficient monte-carlo methods for generating truncated Multivariate Normals \citep{wu2024fast}. We make use of these algorithms through their implementations in the BoTorch library \citep{balandat2020botorch}. We calculate $\chi_t^{*}(a, \bm{x}_t)$  efficiently using the bisection method/binary search that exploits the monotonicity of $P_t(a, b)$ w.r.t to $b$ (see Lemma \ref{lemma:monotonicity} in the Supplementary material). We note that the numerical and monte-carlo methods introduce approximation error. However, we find that these methods perform well in numerical experiments. We leave regret analysis on approximation error to future work.

\section{Robustness to Violations of Concavity}
\label{sup:robust_concavity}

The experiments in the main text focus on settings where the expected reward is concave in the action for every context. We also examine the behavior of CSGP-UCB when this assumption is mildly violated. Let $f_0(a,\bm{x})$ denote one reward function generated from the simulation design in Section~\ref{sec:simu} with $d=25$ and kernel length-scale $\theta=d/10$. We construct perturbed rewards of the form
\[
    f_{\rho}(a,\bm{x}) = f_0(a,\bm{x}) + g_{\rho}(a,\bm{x}),
\]
where $g_{\rho}$ is sampled from a zero-mean Gaussian process with a Gaussian kernel. The kernel for $g_{\rho}$ is anisotropic, with context-space length-scale $\theta_x=4$ and action-space length-scale $\theta_a=0.75$. The smaller action-space length-scale induces substantial variation along the action dimension, so $f_{\rho}$ need not be concave in $a$. We vary the prior standard deviation of $g_{\rho}$ as a proportion
\[
    \rho = \frac{\operatorname{sd}(g_{\rho})}{\operatorname{sd}(f_0)}
\]
of the prior standard deviation of $f_0$ over the input region, with $\rho \in \{0,0.05,0.10,0.25\}$.

We also evaluate a simple robustness modification to the CSGP posterior. Exact concavity corresponds to conditioning on
\[
    \bs{\beta}_j(\bm{x}_{t'}) \leq 0,
    \qquad j=1,\ldots,J-2.
\]
However, the posterior construction and concentration arguments only require an upper-truncated Gaussian posterior. Thus, when mild violations of concavity are plausible, we can replace the hard concavity constraint by the relaxed constraint
\[
    \bs{\beta}_j(\bm{x}_{t'}) \leq c,
    \qquad j=1,\ldots,J-2,
\]
for some scalar coefficient upper bound $c \geq 0$. The choice $c=0$ recovers the exact concavity constraint, while $c>0$ permits bounded positive curvature coefficients in the C-spline representation. The CSGP-UCB algorithm is otherwise unchanged; only the truncation vector $\bs{\nu}$ is modified so that the entries corresponding to the curvature coefficients are set to $c$ rather than $0$.

Table~\ref{tab:concavity_misspecification} reports final cumulative regret after $T=500$ bandit rounds. Parentheses denote $\pm 1$ standard error. We omit DIP-UCB/Thompson from this table because the misspecification exacerbated its previously observed numerical stability issues, leading to substantially worse performance.

\begin{table*}[t]
\centering
\small
\caption{Final cumulative regret after $T=500$ rounds under perturbations that violate concavity. Columns denote the perturbation scale $\rho=\operatorname{sd}(g_{\rho})/\operatorname{sd}(f_0)$. Parentheses give $\pm 1$ standard error. Lower is better.}
\label{tab:concavity_misspecification}

\begin{tabular}{@{}lcccc@{}}
\toprule
\multicolumn{5}{@{}l}{\textbf{CSGP-UCB with relaxed coefficient upper bound}} \\
\midrule
& \multicolumn{4}{c}{Perturbation scale $\rho$} \\
\cmidrule(lr){2-5}
Coefficient upper bound $c$
& $0$ & $0.05$ & $0.10$ & $0.25$ \\
\midrule
$0$
& $101\,(\pm 7)$ & $124\,(\pm 8)$ & $181\,(\pm 9)$ & $539\,(\pm 42)$ \\
$0.1$
& $104\,(\pm 7)$ & $120\,(\pm 8)$ & $167\,(\pm 8)$ & $479\,(\pm 40)$ \\
$0.2$
& $111\,(\pm 8)$ & $115\,(\pm 8)$ & $143\,(\pm 8)$ & $424\,(\pm 37)$ \\
\bottomrule
\end{tabular}

\vspace{0.75em}

\begin{tabular}{@{}lcccc@{}}
\toprule
\multicolumn{5}{@{}l}{\textbf{Unconstrained GP and spline baselines}} \\
\midrule
& \multicolumn{4}{c}{Perturbation scale $\rho$} \\
\cmidrule(lr){2-5}
Method
& $0$ & $0.05$ & $0.10$ & $0.25$ \\
\midrule
GP-Thompson
& $160\,(\pm 8)$ & $169\,(\pm 8)$ & $178\,(\pm 8)$ & $209\,(\pm 9)$ \\
GP-UCB
& $268\,(\pm 10)$ & $271\,(\pm 10)$ & $283\,(\pm 10)$ & $301\,(\pm 11)$ \\
SGP-UCB
& $387\,(\pm 21)$ & $400\,(\pm 21)$ & $422\,(\pm 21)$ & $456\,(\pm 22)$ \\
\bottomrule
\end{tabular}
\end{table*}

The results indicate that CSGP-UCB is robust to minor violations of the concavity condition. For perturbations up to $\rho=0.10$, CSGP-UCB matches or outperforms the strongest benchmark, GP-Thompson. Relaxing the coefficient upper bound improves robustness under misspecification: for example, increasing $c$ from $0$ to $0.2$ reduces regret from $181$ to $143$ when $\rho=0.10$, and from $539$ to $424$ when $\rho=0.25$. However, for perturbations as large as $\rho=0.25$, CSGP-UCB performs poorly relative to GP-Thompson. This is expected, since such a large perturbation removes much of the concave structure of the reward function. In this regime, one should not expect a method designed to exploit concavity to be advantageous.

\section{EXPERIMENTAL DETAILS}
\label{sup:exp_details}

\subsection{Hyper-Parameter Selection and Implementation}

For all algorithms we set the noise parameter to the truth across all experiments. For the NN-UCB algorithm we follow \citep{kassraie2022neural} as set the learning rate for SGD at $\eta = .01$. Note that the use of the SGD optimizer is crucial for the theoretical properties of the algorithm. For NN Thompson \citep{zhang2021neural} we do a grid search on $\nu \in \{1e-1, 1e-2, 1e-3, 1e-4, 1e-5\}$ and report the best performance. For both NN-UCB and NN-Thompson, we follow both papers and use a 1-Layer neural network. We set the number of nodes to the layer 512. For NN-Bootstrap as in \citep{riquelme2018deep, zhou2020neural, zhang2021neural}, we simultaneously train 5 neural networks with a probability of data point inclusion $p = .8$. We train all neural networks for 500 iterations. For CSGP-UCB, use 1000 samples from the elliptical slice sampler of \citet{wu2024fast} to sample monte-carlo samples to calculate medians.

For the Gaussian process based methods, we learn the kernel hyper-
parameters by maximizing the marginal likelihood on the observed data \citep{Rasmussen2006Gaussian}. All GP based methods use Gaussian kernels. 
For neural network based methods, we train/update the appropriate 
network's parameters at every round. For all UCB algorithms, we follow 
\citet{krause2011contextual, srinivas2009gaussian} and choose the 
confidence parameter $\alpha_t$ to satisfy the bounds with probability 
$\delta = 0.1$. To compute $\chi_t^{*}$ we use a bisection algorithm with a tolerance of .025.

Due to computational considerations we also make simplifications in the 
implementation of the algorithms. Following 
\citet{zhou2020neural, zhang2021neural, kassraie2022neural} for the 
neural network based methods, we only calculate the posterior variance 
using the diagonal of the empirical tangent kernel matrix.

For DIP-UCB/Thompson we condition the second derivative to be negative on a uniform grid of 7 points for the current context $\bm{x}_t$ and the 5 nearest observed contexts. We experimented conditioning on a finer grid, but this led to poorer performance due to numerical stability issues. 

For CSGP-UCB, we use a posterior that is only conditioned on the 
concavity information of the current round, $\w{C}_t$ and the 5 nearest observed contexts to the current $\bm{x}_t$. This 
simplification guarantees that the posterior functions obey the concave 
constraint for the current context. Empirical results below demonstrate 
that this works well in practice. For all experiments,  CSGP-UCB uses a cubic C-splines model with 5 uniform knots 
on $[0, 1]$. Cubic splines are the widely used default due to their empirically observed balance between function complexity and smoothness of the resulting curve \citep{hastie2009elements}. We also use $1000$ samples to compute the posterior median. To compute $\chi_t^{*}$ we use a bisection algorithm with a tolerance of .025.

\subsection{Synthetic Benchmarks}
\label{sup:synthetic}
We initialize all algorithms with $n = 25$ training points before starting contextual bandit optimization. We use the Rosenbrock and Dixon-price function as implemented in the BoTorch library \citep{balandat2020botorch}. For the Trid function we follow the specification presented in \cite{simulationlib} and we sample from the hypercube $[-d, d]^{d}$. For all the functions, we choose the action from the first variable and sample the contexts from the last $d-1$ variables.

\subsection{Warfarin Dosing Details}
\label{sup:warfarin_details}

The primary biometric measured for patients on Warfarin is the International Normalized Ratio (INR), a measure of how quickly the blood clots. Within the safe Warfarin dosage regime, the INR has a positive, affine relationship with Warfarin dose \cite{international2009estimation}. In addition patient outcomes are best when their post-dose INR matches the optimal INR given their demographics. Therefore a reasnoble reward function is the absolute difference between the post-dose INR with the patient's optimal INR. As the post-dose INR is an affine function of the dosage, the reward function is concave with respect to the action for every context.

We summarize the test function used by \citet{chen2016personalized} as follows:
\begin{align*}
&f(a, \bm{x}) = 8 + 4 \cos(2\pi\bm{x}_2) - 2 \bm{x}_3 - 8 \bm{x}_5^3 -15 |g(\bm{x}) - 2a|^2 \\
&g(\bm{x}) = 0.6\mathbbm{1}(-0.5 \leq \bm{x}_1 \leq 0.5) + 1.2 \mathbbm{1}(\bm{x}_1 > 0.5) \\
&+ 1.2 \mathbbm{1}(\bm{x}_1 < -0.5) + \bm{x}_4^2 + 0.5 \log(|\bm{x}_7| + 1) - 0.6,
\end{align*}
where $\bm{x} \in [-1, 1]^7$, $a \in [0, 1]$. We define $\w{A}$ as a uniform grid in $[0, 1]$ with $|\w{A}| = 100$. We initialize the data with $n= 25$ data points with contexts and actions drawn uniformly from the domain. We sample the noisy rewards from $y = f(a, \bm{x}) + \epsilon$ where $\epsilon \sim N(0, 0.1)$ as in \citet{chen2016personalized}.

\end{document}